\documentclass{article}
\usepackage[eatrack]{arxiv}
\usepackage{colortbl}
\usepackage{booktabs}           % professional-quality tables
\usepackage{multirow}           % tabular cells spanning multiple rows
\usepackage{amsfonts}           % blackboard math symbols
\usepackage{graphicx}           % figures
\usepackage{duckuments}         % sample images
\usepackage{nicefrac}    
\usepackage{microtype}  
\usepackage{xcolor} 
\usepackage{subfigure}
\usepackage{amsmath, amsfonts, amssymb} 
\usepackage{wrapfig}
\usepackage{ntheorem}
\newtheorem{theorem}{Theorem}% 按 section 编号
\newtheorem{definition}{Definition}

\newtheorem{remark}{Remark}
\newtheorem*{proof}{Proof}
\newtheorem{prop}{Proposition}
\newcommand{\bs}{\hfill $\blacksquare$}

\usepackage[numbers,compress,sort]{natbib}
\title[Dirichlet Energy Enhancement of Graph Neural Networks by Framelet Augmentation]{Dirichlet Energy Enhancement of Graph Neural Networks by Framelet Augmentation}

\author{%
Jialin Chen\\
\institute{Yale University}\\
\email{jialin.chen@yale.edu}\And
Yuelin Wang\\
\institute{Shanghai Jiao Tong University}\\
\email{sjtu\_wyl@sjtu.edu.cn}\And
Cristian Bodnar\\
\institute{University of Cambridge}\\
\email{cb2015@cam.ac.uk}\And
Rex Ying\\
\institute{Yale University}\\
\email{rex.ying@yale.edu}\And
Pietro Li\`{o}\\
\institute{University of Cambridge}\\
\email{pl219@cam.ac.uk}\And
Yu Guang Wang\\
\institute{Shanghai Jiao Tong University}\\
\email{yuguang.wang@sjtu.edu.cn}
}

\begin{document}

\maketitle

\begin{abstract}
Graph convolutions have been a pivotal element in learning graph representations. However, recursively aggregating neighboring information with graph convolutions leads to indistinguishable node features in deep layers, which is known as the over-smoothing issue. The performance of graph neural networks decays fast as the number of stacked layers increases, and the Dirichlet energy associated with the graph decreases to zero as well. In this work, we introduce a framelet system into the analysis of Dirichlet energy and take a multi-scale perspective to leverage the Dirichlet energy and alleviate the over-smoothing issue. Specifically, we develop a \textbf{Framelet Augmentation} strategy by adjusting the update rules with positive and negative increments for low-pass and high-passes respectively. Based on that, we design the \textbf{E}nergy \textbf{E}nhanced \textbf{Conv}olution (\textbf{EEConv}), which is an effective and practical operation that is proved to strictly enhance Dirichlet energy. From a message-passing perspective, EEConv inherits multi-hop aggregation property from the framelet transform and takes into account all hops in the multi-scale representation, which benefits the node classification tasks over heterophilous graphs. Experiments show that deep GNNs with EEConv achieve state-of-the-art performance over various node classification datasets, especially for heterophilous graphs, while also lifting the Dirichlet energy as the network goes deeper.

% In this work, we consider framelet convolutions, which are spectral convolutions on graphs based on framelet transforms that can represent multi-scale graph features through low and high-frequency passes. We develop a framelet augmentation strategy by adjusting the low and high-pass filters with positive and negative increments, respectively. We prove that the resulting approach can increase the Dirichlet energy and thus leads to an Energy Enhanced Convolution (EEConv) for GNNs. 
\end{abstract}

\section{Introduction}\label{sec:intro}
Many types of real-world data, such as social networks, recommendation systems, chemical molecules, contain indispensable relational information, and thus can be naturally represented as a graph. Recently, Graph Neural Networks (GNNs) \cite{kipf2016semi,hamilton2017inductive,velivckovic2017graph} have achieved a myriad of eye-catching performances in multiple applications on graph-structured data. However, for traditional GCNs or other extensions of GNNs, there is a key limitation: the over-smoothing phenomenon, which means that the increase of the model's depth gives rise to the decay of predictive performance. 

There are mainly two types of approaches to enable deep graph neural networks. One is from empirical techniques in graph convolutional layers, like residual connections~\cite{li2019deepgcns, xu2018representation}, weight normalization~\cite{zhao2019pairnorm}, edge dropout~\cite{rong2019dropedge}, etc. The other controls Dirichlet energy to alleviate the over-smoothing phenomenon~\cite{zhou2021dirichlet}. Dirichlet energy is a metric to measure the average distance between connected nodes in the feature space, however, rapidly converges to zero~\cite{cai2020note, oono2019graph} as the number of stacked layers increases. From a spectral perspective, recent works \cite{bo2021beyond,nt2019revisiting,chen2020measuring} discover that graph convolution works well for the case where the low-frequency components are sufficient for prediction, but fails in the scenarios where the high-frequency information is also necessary, which often happens in real-world heterophilous graphs. The failure is due to the denoising effect of graph convolution layers. The unsatisfying performance of GNNs usually stems from insufficient attention to high-frequency components, especially for heterophilous graphs. Therefore, the impact and the potential advantage of multi-scale representation on the over-smoothing issue deserve further exploration. 

Most existing methods that target over-smoothing only consider graph information in the spatial domain and do not characterize the asymptotic behaviors of different frequency components and their different contributions to over-smoothing problem. Besides, empirical techniques lack a theoretical guarantee of the stability or enhancement of the Dirichlet energy. Compared with the previous method that alleviates over-smoothing by controlling the Dirichlet energy, we are the first to theoretically guarantee the enhancement of Dirichlet energy. Furthermore, we emphasize a multi-scale representation for graph-structured data~\cite{xu2019graph,zheng2021framelets,maggioni2008diffusion}, to study asymptotic behaviors of different frequency components.

\paragraph{Present Work} We materialize this idea in a novel \textbf{E}nergy \textbf{E}nhanced \textbf{Conv}olution (\textbf{EEConv}) that can be repeatedly stacked to construct a more robust and deeper GNN architecture by lifting the Dirichlet energy to a higher and steady value. Figure~\ref{fig:cover} illustrates the computational flow of an EEConv layer. We first decompose the graph signal into framelet coefficients (Section~\ref{sec:framelets}), where the global graph structure and all-hop information are embedded by the framelet transform. Then, \textbf{Framelet Augmentation} is applied by modifying the corresponding diagonals of the adjacency matrices for low-pass and high-passes respectively (Section~\ref{sec:3}). Meanwhile, the Dirichlet energy associated with the graph 
is enhanced in this operation. Finally, the framelet coefficients will be reconstructed back to the original size and fed to the non-linear activation.
\begin{figure}[t]
\vspace{-0.5cm}
    \centering
    \includegraphics[width=0.90\textwidth]{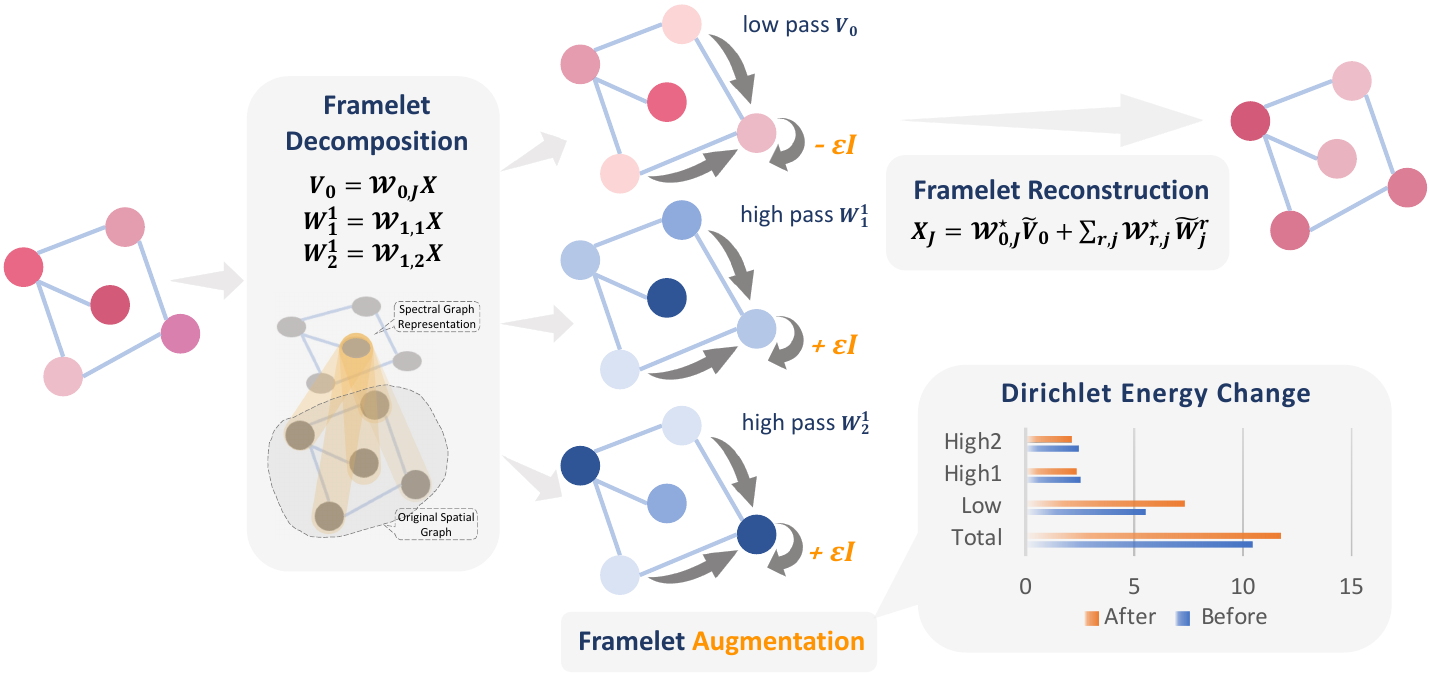}
    \caption{An illustration of the proposed Energy Enhanced Convolution. We first conduct framelet decomposition on the original graph signal (Eq.~\ref{coefficients_0}) and obtain one low-pass and two high-passes. The Framelet Augmentation is applied by adding or subtracting an increment for low and high-passes (Eq.~\ref{modifed_framelet_convolution}). The total Dirichlet energy will be lifted in this process. A framelet reconstruction operator follows to resize the framelet coefficients to the original size.}
    \vspace{-0.8cm}
    \label{fig:cover}
\end{figure}
Our proposed framelet augmentation strategy can be easily extended to other message-passing models with Laplacian-based propagation rules, such as heat diffusion on manifolds. We discuss possible extensions in Section~\ref{discussion}.

We utilize the different roles and contributions of frequency components in graph prediction tasks to control Dirichlet energy. Low-frequency signals can make the representations of adjacent nodes similar and closer, while high-frequency signals make them more distant and distinguishable. Intuitively, we let the model reduce the focus to the low-pass information of the node itself, while increasing the focus to the high-pass components of the neighboring information. Moreover, the decomposability of Dirichlet energy provides us with the feasibility of regulating the energy ratio of each pass. It can be proved that Dirichlet energy is strictly enhanced with the framelet augmentation strategy.

% One can intuitively understand the over-smoothing phenomenon by being aware of the existence of the receptive filed. 
% First, the node embeddings produced by GNNs are determined by the receptive field, which typically includes the node itself and its neighbors. As the number of layers increases and the node features evolve, the receptive field of the node dilates exponentially. Hence, intuitively, we want to carefully control the expansion of the receptive field. Spectral graph theory provides inspiration for how this can be achieved. Some studies show that information from the neighbors is strongly associated with low-frequency components \citep{wu2019simplifying, li2019label}. This encourages us to enhance the Dirichlet energy of the low-pass in the propagation of the network and for balancing weaken high-passes.

% Second, the Dirichlet energy can be decomposed into the energy of information in all passes. In the decomposition of framelet Dirichlet energy, one important observation is that the diagonal matrix is translatable, which implies that the self-enhancement/impairment are allowed to some degree. With the simple trick of introducing a small perturbation, the energy can be separated into two parts: one remains the original part and another depends on the perturbation. We use this as a remedy for over-smoothing by carrying out an augmentation for the framelet convolution where the low-pass and high-passes in the framelet transforms are added/subtracted by a small perturbation. As a result, the Dirichlet energy does not decay. 

To this end, the contributions of this work are threefold: \textbf{(1)} We perform a systematic analysis of the Dirichlet energy based on the framelet system and propose a novel Framelet Augmentation strategy to enhance the Dirichlet energy.
\textbf{(2)} We theoretically prove the different asymptotic behaviors and Dirichlet energy of low-pass and high-passes during the feature propagation, and validate them through sufficient experiments. \textbf{(3)} We carry out experiments to verify the effectiveness of Framelet Augmentation and demonstrate that the proposed approach achieves outstanding performance on real-world node classification tasks, especially for heterophilous graphs.

\section{Background and Preliminaries}\label{sec:framelets}
\subsection{Framelet Analysis on Graph}
Wavelet analysis on manifolds provides a powerful multi-scale representation tool for geometric deep learning. In this paper, we mainly focus on tight framelets on a graph \cite{zheng2021framelets,dong2017sparse,zhang2018end}, which is a multi-scale affine system. 
\begin{figure}[!t]
    \vspace{-0.3cm}
    \centering
    \includegraphics[width=\textwidth]{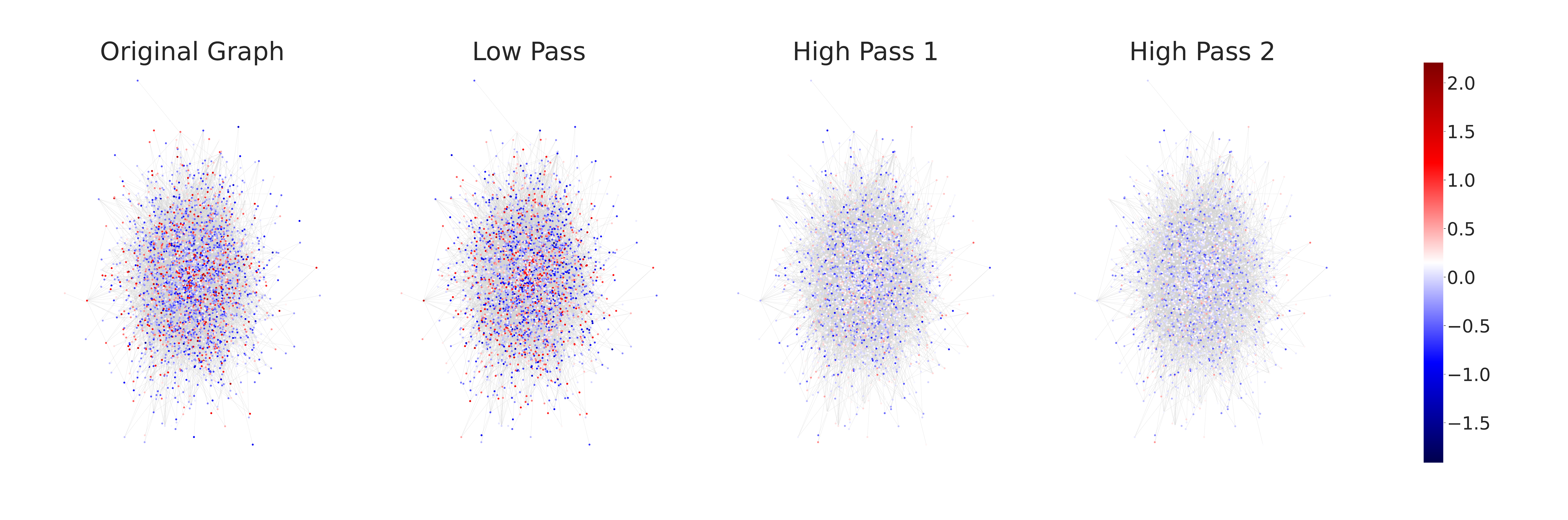}
    \vspace{-0.5cm}
    \caption{Visualization of framelet coefficients for node classification task on Cora. From left to right we show the original graph, low-pass, high-pass 1 and high-pass 2 respectively. The projected value is the first principal component of the high-dimensional features.}
    \label{fig:framelet}
    \vspace{-0.3cm}
\end{figure}
For a graph $\mathcal{G}$ with $N$ nodes and graph Laplacian $\Delta$, let $U=[\boldsymbol{u}_1,\cdots, \boldsymbol{u}_N]$ be the matrix of eigenvectors of $\Delta$, whose transpose is used for the Graph Fourier Transform, and $\Lambda=\textup{diag}(\lambda_1,\cdots,\lambda_N)$ be the diagonal matrix of the eigenvalues. Framelets over the graph is generated by a set of \textit{scaling functions} $\Phi = \{\alpha;\beta^{(1)},\cdots,\beta^{(n)}\}\subset L_1(\mathbb{R})$ associated with a \textit{filter bank} $\eta = \{a; b^{(1)},\cdots b^{(n)}\}$, satisfying $\widehat{\alpha}(2\xi)=\widehat{a}(\xi)\widehat{\alpha}(\xi)$ and $\widehat{\beta^{(r)}}(2\xi)=\widehat{b^{(r)}}(\xi)\widehat{\alpha}(\xi)$, for any $\xi\in \mathbb{R}$, where $\widehat{h}(\xi)$ denotes the \textit{Fourier transform} of $h$, which is defined by $\widehat{h}(\xi) :=\sum_{k\in \mathbb{Z}}h(k)e^{-2\pi ik\xi}$. $n$ denotes the number of high-pass filters. $\varphi_{j,p}(v)$ and $\psi_{j,p}^{r}(v)$ are the \textit{low-pass} and \textit{high-pass} framelets at node $v$ associated to node $p$ for \textit{scale level} $j \in\{1,\cdots, J\}$ respectively, which is defined by
\begin{equation*}
    \varphi_{j,p}(v) = \sum_{l=1}^N \widehat{\alpha}\left(\frac{\lambda_{l}}{2^j}\right)u_l(p)u_l(v);\quad 
    \psi_{j,p}^{r}(v) = \sum_{l=1}^N \widehat{\beta}^{(n)}\left(\frac{\lambda_{l}}{2^j}\right)u_l(p)u_l(v), \; r=1,\dots, n.
\end{equation*}
Therefore, the framelet transforms actually take into account the global information and all the hops of the graph into its multi-scale representations. The low-pass and high-pass framelets distill the coarse-grained and fine-grained information of graph signals.

The \textit{framelet coefficients} $V_0, W_j^r \in \mathbb{R}^{N\times d}$ are defined as the inner-product of the framelet and the graph signal $X\in \mathbb{R}^{N\times d}$, where $d$ denotes the feature dimension. The size of $V_0, W_j^r$ is the same as the graph signal (node features) $X$.
\begin{equation}
    V_0 =\left \langle\varphi_{0,\cdot}, X\right \rangle = U^\top \widehat{\alpha}\bigl(\frac{\Lambda}{2}\bigr)UX\hbox{~~and~~} W_j^r= \left \langle\psi_{j,\cdot}^r, X\right \rangle=U^\top\widehat{\beta^{(r)}}\bigl(\frac{\Lambda}{2^{j+1}}\bigr)UX,
    \label{coefficients_0}
\end{equation}
% where the scaling functions on $\mathcal{G}$ are as follows,
% \begin{equation*}
%     \widehat{\alpha}\bigl(\frac{\Lambda}{2^{j+1}}\bigr)= \mathrm{diag}\left(\widehat{\alpha}\bigl(\frac{\lambda_1}{2^{j+1}}\bigr),\cdots,\widehat{\alpha}\bigl(\fracz x q{\lambda_N}{2^{j+1}}\bigr)\right),\quad
%     \widehat{\beta^{(r)}}\bigl(\frac{\Lambda}{2^{j+1}}\bigr)
%     =\mathrm{diag}\left(\widehat{\beta^{(r)}}\bigl(\frac{\lambda_1}{2^{j+1}}\bigr),\cdots,\widehat{\beta^{(r)}}\bigl(\frac{\lambda_N}{2^{j+1}}\bigr)\right).
% \end{equation*}
Let $\mathcal{W}_{r,j}$ denote the decomposition operators given by $V_0 = \mathcal{W}_{0,J}X$ and $W_j^r = \mathcal{W}_{r,j}X$. Then according to Eq.~\ref{coefficients_0}, we obtain the framelet transform matrices for decomposition:
    \begin{equation}
    \begin{aligned}
        \mathcal{W}_{0,J} &=U^\top\widehat{a}(2^{-K+J-1} \Lambda)\cdots\widehat{a}(2^{-K}\Lambda)U:=U^\top \Lambda_{0,J}U,\\
        \mathcal{W}_{r,1} &= U^\top \widehat{b^{(r)}}(2^{-K}\Lambda)U:=U^\top \Lambda_{r,1}U,\\
        \mathcal{W}_{r,j}&=U^\top \widehat{b^{(r)}}(2^{-K+j-1}\Lambda)\widehat{a}(2^{-K+j-2}\Lambda)\cdots\widehat{a}(2^{-K}\Lambda)U:=U^\top \Lambda_{r,j}U.
    \end{aligned}
    \label{operators}
\end{equation}
Here, $K$ is a sufficiently large value such that the Laplacian's biggest eigenvalue $\lambda_{max}\leq 2^K\pi$. We use Haar-type filters, a classic multi-scale system with acceptable computational cost in our implementation. With Haar-type filters, we have $\widehat{\alpha}(\frac{\Lambda}{2}) = \text{cos}(\frac{\Lambda}{8})\text{cos}(\frac{\Lambda}{16})$, $\widehat{\beta}(\frac{\Lambda}{2}) = \text{sin}(\frac{\Lambda}{8})\text{cos}(\frac{\Lambda}{16})$ and $\widehat{\beta}(\frac{\Lambda}{4}) = \text{sin}(\frac{\Lambda}{16})$ to construct a framelet system of 2 scale level ($j=1,2$) and 1 high-pass filter ($r=1$). Thus, we obtain one low-pass ($V_0$) and two high-passes ($W_1^1, W_2^1$). Figure~\ref{fig:framelet} shows the scattering plots of the principal component of framelet coefficients on the Cora dataset. We can observe that the low-pass provides an approximation of the original graph signal while the high-passes distill the detail information.
\paragraph{Energy Gap}
The magnitude of the high-passes coefficients is usually much smaller than the low-pass. With $L_2$ norm as the energy of signals, we can prove that the sum of high-pass energies is less than that of the low-pass, or precisely, $\|W_1^1\|^2+\|W_2^1\|^2\leq \|V_0\|^2$. See the proof and $L_2$ norm statistical results in Figure~\ref{fig:L2_norm} in Appendix~\ref{energy_gap}. This motivates us to consider the energy imbalance between low and high-passes. 

% \begin{equation}
%     \begin{aligned}
%         \mathcal{W}_{0,J} &\approx U^\top\mathcal{T}_0(2^{-K+J-1}\Lambda)\cdots\mathcal{T}_0(2^{-K}\Lambda)U=\mathcal{T}_0(2^{-K+J-2}\Delta)\cdots\mathcal{T}_0(2^{-K}\Delta),\\
%         \mathcal{W}_{r,1} &\approx U^\top\mathcal{T}_r(2^{-K}\Lambda)U=\mathcal{T}_r(2^{-K}\Delta),\\
%         \mathcal{W}_{r,j}&\approx U^\top\mathcal{T}_r(2^{-K+j-1}\Lambda)\mathcal{T}_0(2^{-K+j-2}\Lambda)\cdots\mathcal{T}_0(2^{-K}\Lambda)U\\
%         &=\mathcal{T}_r(2^{-K+j-1}\Delta)\mathcal{T}_0(2^{-K+j-2}\Delta)\cdots\mathcal{T}_0(2^{-K}\Delta).
%     \end{aligned}
% \end{equation}
\subsection{Dirichlet Energy}
Dirichlet Energy measures the degree of over-smoothing phenomenon, by calculating the average representation distance between connected nodes. Over-smoothing representations will produce a small value of Dirichlet Energy and cause the decay of the model's prediction performance. Let $\Tilde{A}=A+I_N$ be the adjacency matrix of the original graph with self-loops. $\Tilde{D}$ is the diagonal degree matrix associated with $\Tilde{A}$. With the augmented adjacency matrix $\widehat{A} =\Tilde{D}^{-\frac{1}{2}}\Tilde{A}\Tilde{D}^{-\frac{1}{2}}$, augmented normalized Laplacian~\citep{wu2019simplifying} of the input graph is defined as $\Tilde{\Delta}=I_N-\widehat{A} = I_N-\Tilde{D}^{-\frac{1}{2}}\Tilde{A}\Tilde{D}^{-\frac{1}{2}}$.
\begin{definition}
    Dirichlet energy $E(X)$ of the signal $X\in \mathbb{R}^{N\times 1}$ on the graph $\mathcal{G}(V,E)$ is defined as 
    \vspace{-0.2cm}
    \begin{equation*}
        E(X) = X^\top\Tilde{\Delta}X = \frac{1}{2}\sum_{(i,j)\in E} w_{ij}\left(\frac{X_i}{\sqrt{1+d_i}}-\frac{X_j}{\sqrt{1+d_j}}\right)^2,
    \end{equation*}
    where $\Tilde{\Delta}$ is the augmented normalized Laplacian. Similarly, for multiple channels the Dirichlet energy is defined as \textup{trace}$(X^\top\Tilde{\Delta}X)$.
\end{definition}

For the propagation rule of GCN \cite{kipf2016semi}:
$ H^{(l+1)}=\sigma(\Tilde{D}^{-\frac{1}{2}}\Tilde{A}\Tilde{D}^{-\frac{1}{2}}H^{(l)}W^{(l)})$, where $H^{(l)}$ is the feature representations of the $l$-th layer, $W^{(l)}$ is the $l$-th layer weight matrix, the following Theorem~\ref{thm:convergence}~\cite{oono2019graph} implies the convergent behavior of node features. The subspace that node features converge to is formulated with the bases of the eigenspace of graph Laplacian~\cite{cai2020note, oono2019graph}. 
We first clarify the notations for Theorem~\ref{thm:convergence}. Let $K$ be the null space of the graph laplacian $\Tilde{\Delta}$. The subspace $\mathcal{M}$ is defined by $\mathcal{M}=K\otimes\mathbb{R}^d
=\bigl\{\sum_{m=1}^M e_m\otimes w_m | w_m\in \mathbb{R}^d; e_m\in K\bigr\}\subseteq\mathbb{R}^{N\times d}$. $d$ is the feature dimension of the graph signal. The distance between graph signal $X\in \mathbb{R}^{N\times d}$ and subspace $\mathcal{M}$ is defined as $d_{\mathcal{M}}(X)=\textup{inf}_{m\in\mathcal{M}}\{\|X-m\|_F\}$, where $F$ denotes the Frobenius norm.

\begin{theorem}\textsc{\cite{oono2019graph}}
For GCN models, we have that $d_{\mathcal{M}}(H^{(l+1)})\leq s_l\lambda d_{\mathcal{M}}(H^{(l)})$, where $\lambda$ is the second largest eigenvalue of the augmented adjacency matrix $\widehat{A}$ and $s_l$ is the supremum of all singular values of the $l$-th layer weight matrix $W^{(l)}$.
\label{thm:convergence}
\end{theorem}
The convergence rate of the distance between node features and the subspace is positively related to the eigenvalues of the $\widehat{A}$ \cite{huang2020tackling}, generating the consistent feature representations of nodes. over-smoothing is especially detrimental in heterophilous graph tasks, where adjacent nodes are more likely to have different labels. Thus, too similar feature representations between connected nodes (but most likely with different labels) lead to the failure of GNNs in these tasks.  

\section{Framelet Augmentation Strategy}\label{sec:3}
\subsection{Framelet Convolution} 
With the above Laplacian-based framelet transforms, we develop the framelet (graph) convolution similar to the graph convolution (GCNConv \cite{kipf2016semi}) as follows:
\begin{equation}
        H^{(l+1)}_{r,j}=\sigma(\Tilde{D}^{-\frac{1}{2}}\Tilde{A}\Tilde{D}^{-\frac{1}{2}}\mathcal{W}_{r,j}H^{(l)}W^{(l)}_{r,j})\qquad
        H^{(l+1)} = \mathcal{V}(H^{(l+1)}_{0,J};H^{(l+1)}_{1,1},\cdots,H^{(l+1)}_{n,J}),
    \label{framelet_convolution}
\end{equation}
where $(r,j)\in\{(r,j)|r=1,\cdots,n;j=1,\cdots,J\}\cup\{(0,J)\}$, $W_{r,j}^{(l)}$ is the trainable weight matrix corresponding to the $l$-th layer and the $(r,j)$-th pass, $\mathcal{V}$ is the framelet reconstruction operator given by $X_J =\mathcal{V}(V_0,W_{1,1},\cdots, W_{n,J}) = \mathcal{W}_{0,J}^\star V_0+\sum_{r=1}^n\sum_{j=1}^J\mathcal{W}_{r,j}^\star W_{r,j}$, where the superscript $\star$ indicates the conjugate transpose of the matrix. We can observe that $\mathcal{V}$ reconstructs the low-pass and high-pass coefficients back to the original size.

Compared with the existing framelet graph model, UFG \cite{zheng2021framelets}, which is a filter learning method in the frequency domain, our framelet graph convolution inherits the message-passing pattern and generalizes that to multi-scales representation systems.

\subsection{Framelet Dirichlet Energy}
Based on Section~\ref{sec:framelets}, we define framelet Dirichlet energy for low-pass and high-passes signals.
\begin{equation}\label{eq:dirichlet energy framelet}
    E_{0,J}(X)=(\mathcal{W}_{0,J}X)^{\top} \Tilde{\Delta}(\mathcal{W}_{0,J}X);\qquad
    E_{r,j}(X)=(\mathcal{W}_{r,j}X)^{\top} \Tilde{\Delta}(\mathcal{W}_{r,j}X).
\end{equation}
The total framelet Dirichlet energy is then defined as the sum of Dirichlet energy in each pass: $$E_{total}(X) = \sum_{r=1}^{n}\sum_{j=1}^J E_{r,j}(X) + E_{0,J}(X).$$
\begin{prop}
The Dirichlet energy is conserved under framelet decomposition:
\begin{equation}
    E_{total}(X)=\sum_{r=1}^{n}\sum_{j=1}^J E_{r,j}(X) + E_{0,J}(X)=E(X).
\end{equation}
\label{DE_conservation_1}
\end{prop}
\begin{remark}
\vspace{-0.2cm}
The Dirichlet energy components $E_{r,j}(X):= X^\top \mathcal{W}_{r,j}^\top\Tilde{\Delta} \mathcal{W}_{r,j}X$ are controlled by $\Lambda_{r,j}^2$, the diagonal matrix given in Eq.~\ref{operators}, where $(r,j)\in\{(r,j)|r=1,\cdots,n;j=1,\cdots,J\}\cup\{(0,J)\}$.
\label{remark_1}
\end{remark}
Proposition~\ref{DE_conservation_1} and Remark~\ref{remark_1} guarantee that we can decompose the signal into low-pass and high-passes and precisely control their Dirichlet energy proportions, without changing the overall Dirichlet energy. See the proofs in Appendix~\ref{DE_proof}.
\subsection{Dirichlet Energy Enhancement Architecture}\label{sec:framelet_aug}
\paragraph{Energy Enhanced Convolution}
The key idea to tackle the over-smoothing issue is to preserve the Dirichlet Energy and avoid its exponential decay to zero with respect to the number of layers. Motivated by this, we propose a \textbf{Framelet Augmentation} strategy, using the properties of multi-scale framelets to enhance the overall Dirichlet energy. To take advantage of the energy gap between low-pass and high-passes, we decouple the low-pass and high-passes propagation and modify the low-pass adjacency matrix $\widehat{A}^L$ and high-pass adjacency matrix $\widehat{A}^H$ separately. The augmented normalized Laplacian $\Tilde{\Delta}$ is changed correspondingly, since $\Tilde{\Delta} = I_N-\widehat{A}$. $\epsilon$ controls the level of self-enhancement and impairment, which is a hyper-parameter in the implementation.
\begin{alignat}{2}
\widehat{A}^L&=\widehat{D}^{-\frac{1}{2}}(\Tilde{A}-\epsilon I) \widehat{D}^{-\frac{1}{2}}=\Tilde{A} - \epsilon \widehat{D}^{-1},&\quad  \widehat{A}^H&=\widehat{D}^{-\frac{1}{2}}(\Tilde{A}+\epsilon I) \widehat{D}^{-\frac{1}{2}}=\Tilde{A} + \epsilon \widehat{D}^{-1}.\label{definition:adjacency}\\  
\Tilde{\Delta}^L&=I_N-\widehat{A}^L = \Tilde{\Delta}+\epsilon \widehat{D}^{-1},&\quad \Tilde{\Delta}^H&=I_N-\widehat{A}^H =\Tilde{\Delta}-\epsilon \widehat{D}^{-1}.
\label{definition:laplacian}
\end{alignat}
Next, with the modified adjacency matrices in the low-pass and high-passes, we have the following layer-wise propagation rule of \textbf{Energy Enhanced Convolution}:
\begin{equation}
    \begin{aligned}
        H^{(l+1)}_{0,J}&=\sigma(\widehat{A}^L\mathcal{W}_{0,J}H^{(l)}W^{(l)}_{0,J})\\
        H^{(l+1)}_{r,j}&=\sigma(\widehat{A}^H\mathcal{W}_{r,j}H^{(l)}W^{(l)}_{r,j}),\quad \hbox{for} (r,j)\in \{(r,j)|r=1,\cdots,n;j=1,\cdots,J\}\\
        H^{(l+1)} &= \mathcal{V}(H^{(l+1)}_{0,J};H^{(l+1)}_{1,1},\cdots,H^{(l+1)}_{n,J})
    \end{aligned}
    \label{modifed_framelet_convolution}
\end{equation}
\paragraph{Dirichlet Energy Enhancement}
The low-pass component $E_{0,J}^{\epsilon}$ and high-pass components $E_{r,j}^{\epsilon}$ of Dirichlet energy with modified Laplacian are defined correspondingly as Eq.~{\ref{modifed_E}}.
\begin{equation}
\begin{aligned}
    E_{0,J}^{\epsilon}(X,)&=(\mathcal{W}_{0,J}X)^{\top} \Tilde{\Delta}^L(\mathcal{W}_{0,J}X)=(\mathcal{W}_{0,J}X)^{\top} (\Tilde{\Delta}+\epsilon \widehat{D}^{-1})(\mathcal{W}_{0,J}X)\\
    E_{r,j}^{\epsilon}(X)&=(\mathcal{W}_{r,j}X)^{\top} \Tilde{\Delta}^H(\mathcal{W}_{r,j}X)=(\mathcal{W}_{r,j}X)^{\top} (\Tilde{\Delta}-\epsilon \widehat{D}^{-1})(\mathcal{W}_{r,j}X)
\end{aligned}
\label{modifed_E}
\end{equation}
The following theorem guarantees that we can obtain a strict enhancement of Dirichlet energy during the feature propagation by Framelet Augmentation.
\begin{theorem}
The total framelet Dirichlet energy is increased with low-pass adjacency matrix $\widehat{A}^L$ and high-pass adjacency matrix $\widehat{A}^H$ when $\epsilon >0$, i.e., $E_{total}^{\epsilon}(X)=\sum_{r=1}^{n}\sum_{j=1}^J E_{r,j}^{\epsilon}(X)+E_{0,J}^{\epsilon}(X)>E_{total}(X) = E(X)$. 
\label{prop:energy_enhance}
\end{theorem}
The proof of Theorem~\ref{prop:energy_enhance} is given in Appendix~\ref{append:enhancement_proof}. $\epsilon>0$ indicates strengthening self-connection to the high-passes and weakening that to the low-pass.
\subsection{Computational Complexity}
To reduce the computational complexity caused by eigendecomposition for graph Laplacians, we use Chebyshev polynomials to approximate the framelet decomposition in our implementation. The framelet transform is then equivalent to left-multiplying a specific transformation matrix. We stack the transformation matrices to obtain a tensor-based framelet transform with the computational complexity of $\mathcal{O}(N^2(nJ+1)d)$. $N$ is the number of nodes, $d$ is the feature dimension, $n$ is the number of high-pass filters and $J$ is the scale level of the low-pass. See Appendix~\ref{append:complexity} for an empirical study of the complexity.

\subsection{Asymptotic Behavior of Framelet Components}
We can understand the effect of Framelet Augmentation in terms of the asymptotic behavior of framelet components. Framelet Augmentation helps to increase the weight of the high-frequency information of the node itself during the message passing process. It also reduces the proportion of high-pass component $E^\epsilon_{r,j}$ in the total Dirichlet energy, giving rise to the closer distances between the high-frequency components of node representations. The following proposition implies the asymptotic behaviors of low-pass and high-pass signals during the learning process.

\begin{prop}
    Let $A$ be an $n\times n$ augmented adjacency matrix, which is (symmetric) positive definite. $\lambda_k(A)$ is the $k$-th largest eigenvalue of $A$ ($k=1,2,\cdots,n$). Let $A(\epsilon) = A + \epsilon D$, where $D$ is a positive diagonal matrix. Then $\lambda_k(A(\epsilon))$ increases monotonically with $\epsilon$ and the following relation holds:
    $$\lambda_k(A^L)\leq \lambda_k(A)\leq \lambda_k(A^H)\qquad (\epsilon\geq 0),$$
    where $A^L$ and $A^H$ are low-pass and high-passes adjacency matrices as defined in Eq.~\ref{definition:adjacency}.
    \label{prop:lambda}
\end{prop}
See proof of Proposition~\ref{prop:lambda} and empirical study of asymptotic behavior of each pass (Figure~\ref{fig:Distance}) in Appendix~\ref{append:lambda}. According to Theorem~\ref{thm:convergence}, adding $\epsilon I$ as a self-connectivity term in the high-pass increases its second largest eigenvalue of the adjacency matrix, leading to the slower convergence to the subspace and impeding the over-smoothing with an overall enhanced Dirichlet energy.

\subsection{Equivariance of Framelet Convolution}
Equivariance and Invariance are important properties for graph neural networks and we have the following Proposition. 
\begin{prop}
    An EEConv layer is permutation equivariant.
    \label{prop:equivariance}
\end{prop}
See the proof in Appendix~\ref{appendix:equivariance}. The framelet transforms are naturally generalized from the graph Fourier transform, therefore, framelet decomposition does not destroy the permutation invariance of graph neural networks. In hence, we can stack multiple EEConv layers, followed by a final invariant read-out function to obtain an equivariant deep graph neural network.

\section{Experiments}\label{sec:experiment}
To verify the effectiveness of Framelet Augmentation strategy, we evaluate: (A) Dirichlet energy behavior with respect to the number of layers and homophily level of the graphs and (B) the model's performance of node classification over real-world datasets with different homophily levels and the change of performance with respect to the number of layers.
\subsection{Dirichlet Energy Behavior}
\begin{figure}[htbp]
\vspace{-0.7cm}
\centering
\subfigure[\scriptsize Dirichlet energy on Cora, H=0.81]{
\begin{minipage}[t]{0.32\textwidth}
\centering
\hspace{-0.5cm}
\includegraphics[width=1.1\textwidth]{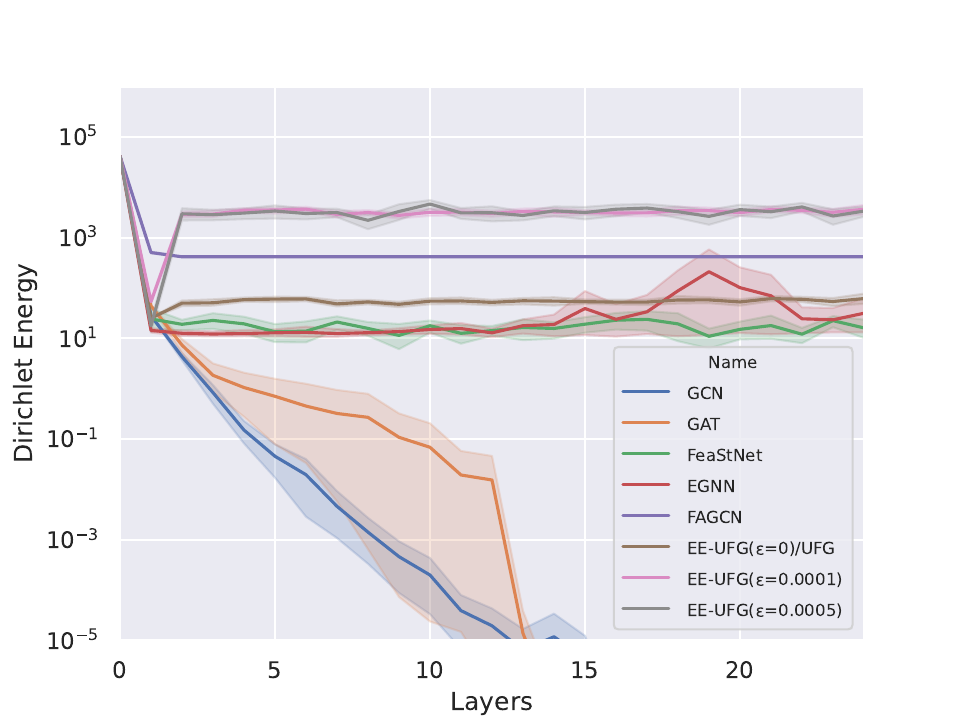}
%\caption{fig1}
\end{minipage}%
}%
\subfigure[\scriptsize Framelet Dirichlet energy]{
\begin{minipage}[t]{0.32\textwidth}
\hspace{-0.5cm}
\includegraphics[width=1.1\textwidth]{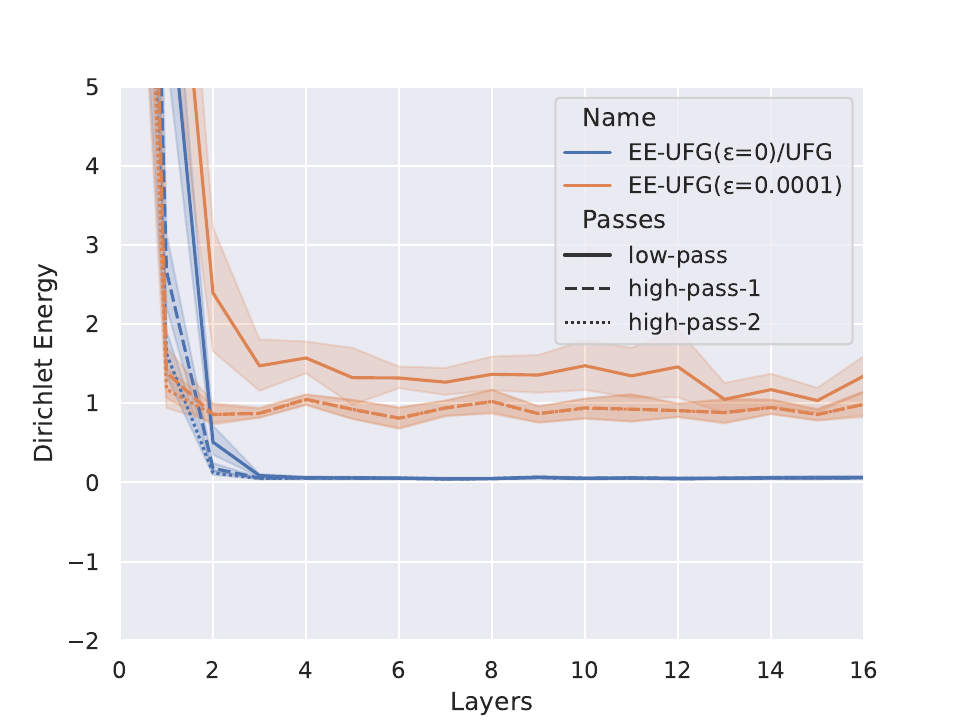}
%\caption{fig2}
\end{minipage}
}%
\subfigure[\scriptsize Dirichlet energy with $p/q$ (\textbf{3} Layers)]{
\begin{minipage}[t]{0.32\textwidth}
\centering
\hspace{-0.5cm}
\includegraphics[width=1.1\textwidth]{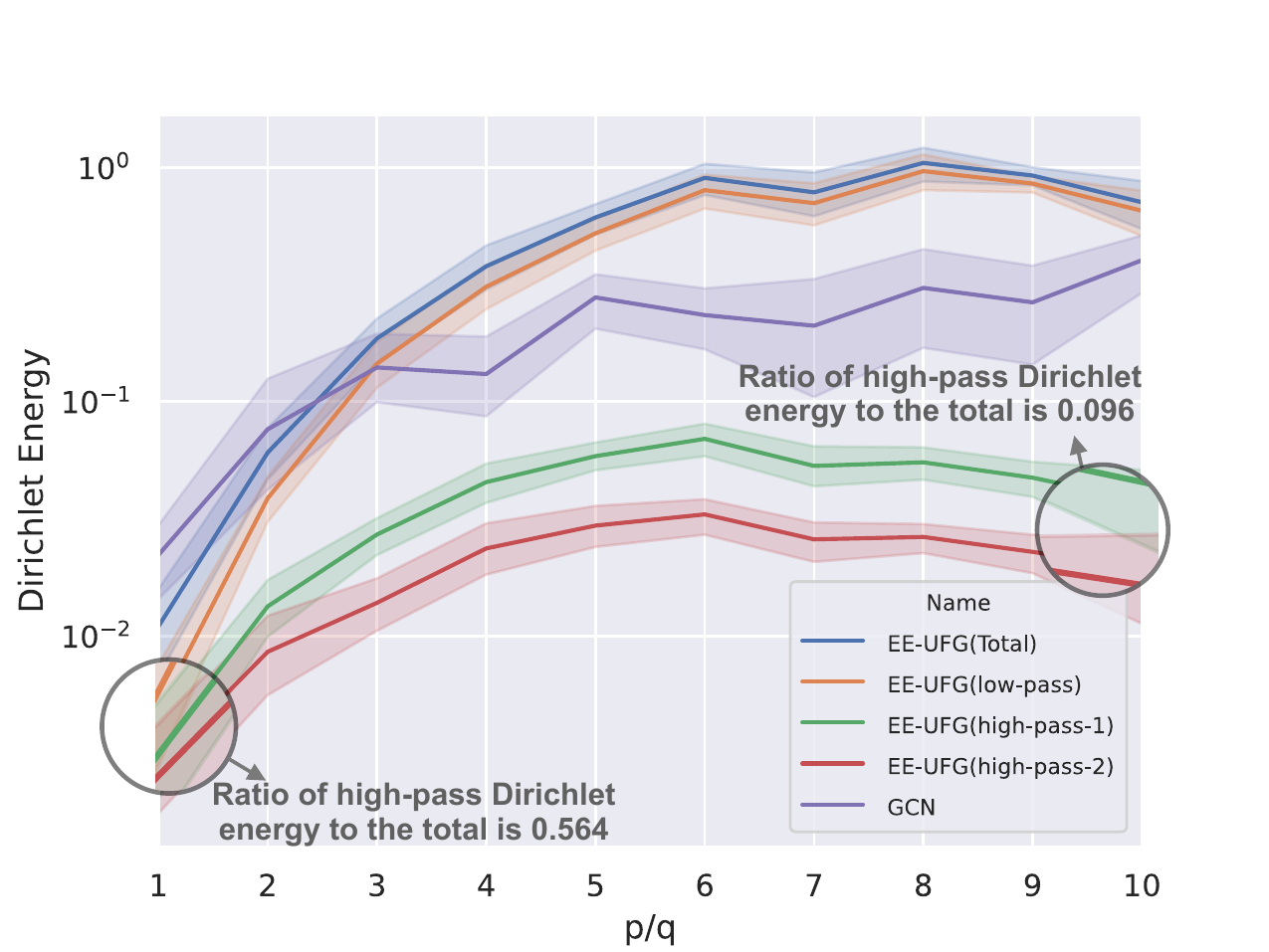}
%\caption{fig1}
\end{minipage}%
}%
\\
\vspace{-0.45cm}
\subfigure[\scriptsize Dirichlet energy on Chameleon, H=0.23]{
\begin{minipage}[t]{0.32\textwidth}
\centering
\hspace{-0.5cm}
\includegraphics[width=1.1\textwidth]{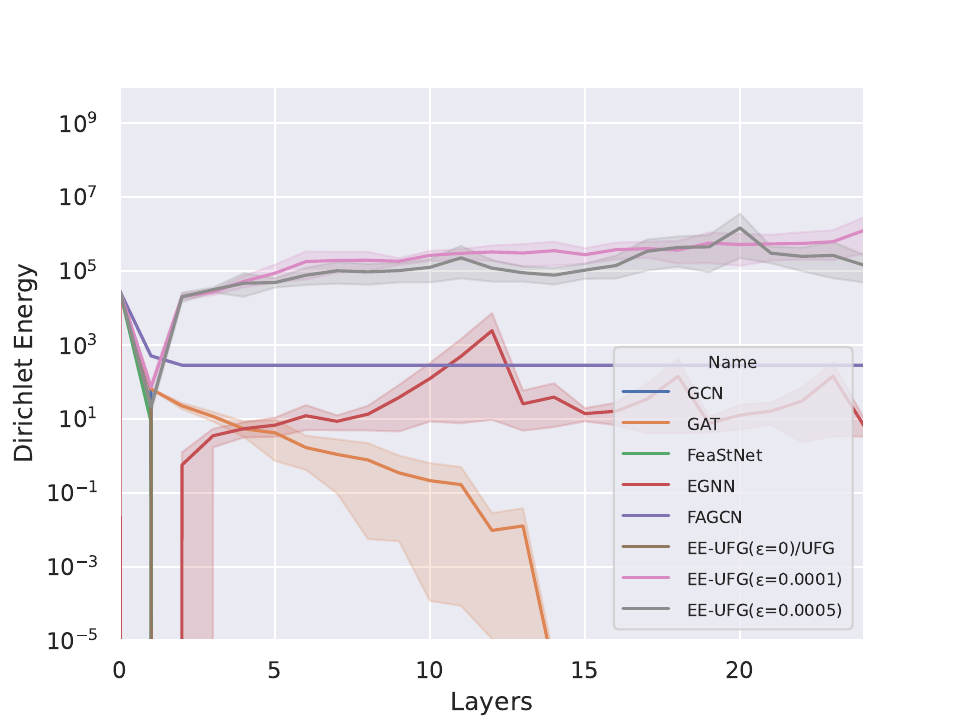}
%\caption{fig1}
\end{minipage}%
}%
\subfigure[\scriptsize Framelet Dirichlet energy proportions]{
\begin{minipage}[t]{0.32\textwidth}
\centering
\hspace{-0.5cm}
\includegraphics[width=1.1\textwidth]{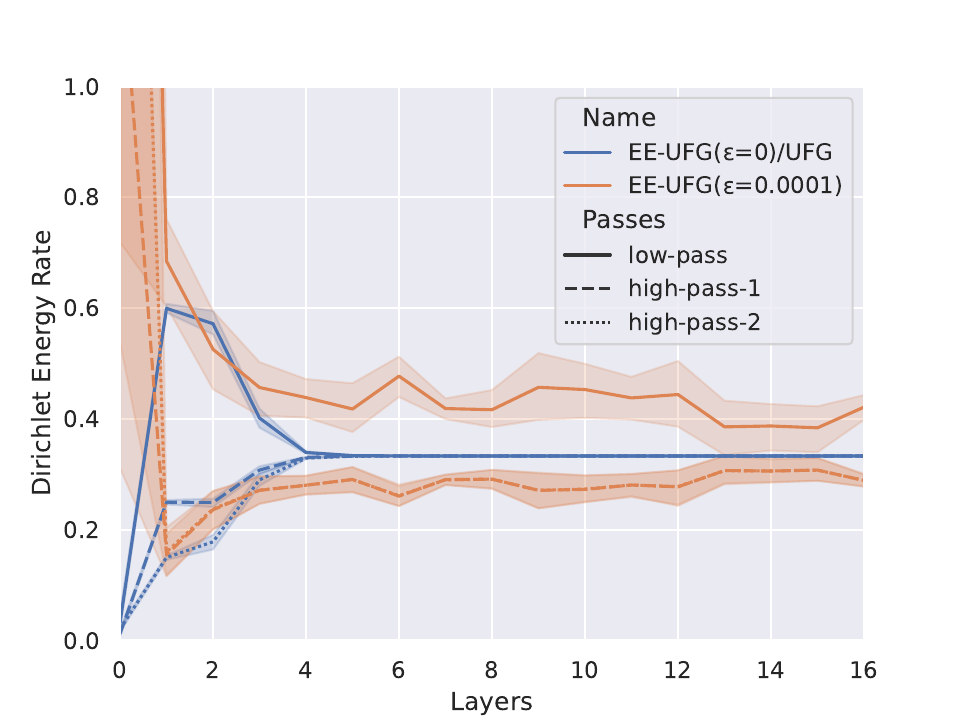}
%\caption{fig1}
\end{minipage}%
}%
\subfigure[\scriptsize Dirichlet energy with  $p/q$ (\textbf{8} Layers)]{
\begin{minipage}[t]{0.32\textwidth}
\centering
\hspace{-0.5cm}
\includegraphics[width=1.1\textwidth]{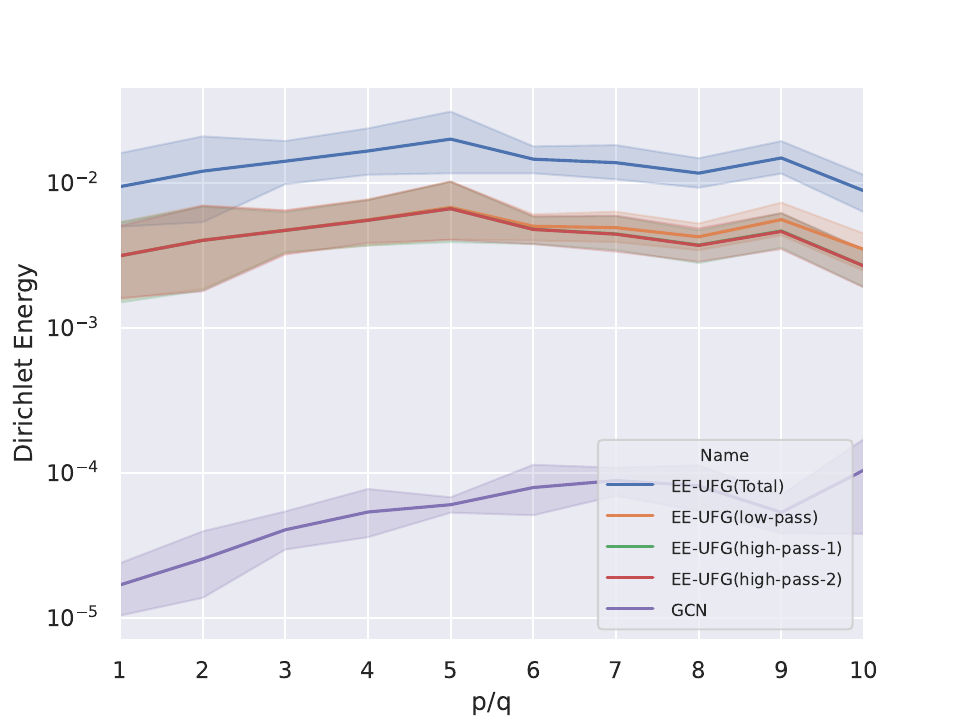}
%\caption{fig1}
\end{minipage}%
}%
\caption{(a) Layer-wise Dirichlet energy with different models on Cora dataset. (b) Layer-wise framelet Dirichlet energy components. (c) Dirichlet energy over graphs with different $p/q$ ratios by 3-layers model. (d) Layer-wise Dirichlet energy with different models on Chameleon dataset. (e) Layer-wise framelet Dirichlet energy ratios to the total Dirichlet energy. (f) Dirichlet energy over graphs with different $p/q$ ratios by 8-layers model.}
\label{fig:DE_Layers}
\vspace{-0.75cm}
\end{figure}
We select two real-world datasets to verify the effect of framelet augmentation for alleviating the exponentially decay of Dirichlet energy: Cora which is a relatively homophilous graph dataset with a homophily level of $0.81$\footnote{We use the homophily level defined in~\cite{pei2020geom}.}, and Chameleon with lower homophily level of $0.23$.  We show the layer-wise (logarithm of) Dirichlet energy during the feature propagation through GCN \cite{kipf2016semi}, GAT \cite{velivckovic2017graph}, FeaStNet \cite{verma2018feastnet}, EGNN~\cite{zhou2021dirichlet}, FAGCN \cite{bo2021beyond} and our Energy-enhanced UFG (EE-UFG) in Figure~\ref{fig:DE_Layers}(a) and Figure~\ref{fig:DE_Layers}(d). When $\epsilon$ is selected as 0, our EE-UFG is equivalent to the spatial version of UFG \cite{zheng2021framelets}. We can observe that the Dirichlet energy usually decays fast to zero with respect to the number of layers in the regular graph convolutional models. In heterophilous graphs, some existing models suffer from Dirichlet energy instability (e.g., EGNN). With framelet augmentation, EE-UFG can lift the Dirichlet energy to a higher and steady state compared with other baseline models that target over-smoothing issues, thus making the output features distinguishable.

In Figure~\ref{fig:DE_Layers}(b) and Figure~\ref{fig:DE_Layers}(e), we plot the absolute value of framelet Dirichlet energy components and the ratios to the total Dirichlet energy with respect to the number of layers. For the case where $\epsilon=0$, the low-pass component and high-passes components quickly decay to zero. One effect of the message-passing convolutions is that the proportions of high-passes and low-pass in the total Dirichlet energy tend to be the same, which means the message-passing mechanism automatically eliminates the energy gap between high-frequency and low-frequency components. However, the EE-UFG with framelet augmentation not only enhances the overall Dirichlet energy, but also decouples the low-pass and high-passes signals and preserves the energy gap during the feature propagation. 

In Figure~\ref{fig:DE_Layers}(c) and Figure~\ref{fig:DE_Layers}(f), we use the Stochastic Block Model (SBM) to randomly generate undirected graphs with 100 nodes that are divided into 2 classes. The node features are sampled from Gaussian distribution $\mathcal{N}(0.5,1)$ and $\mathcal{N}(-0.5,1)$ for two classes. Edges are generated as Bernoulli random variables following intra-class connection probability $p$ and inter-class connection probability $q$. The ratio of $p/q$ depicts the homophily of the graph. The higher the ratio, the more homophilous the graph is. Figure~\ref{fig:DE_Layers}(c) shows the Dirichlet energy with different $p/q$ ratios. When using 3-layer models, which is empirically the optimal number of layers for classic GCN, Dirichlet energy decreases as the homophily decreases. This implies that heterophily and over-smoothing are correlated. Moreover, the ratio of high-passes Dirichlet energy to the total is $0.564$ when $p/q=1$ and decreases to $0.096$ when $p/q=10$. The high passes account for a larger proportion of the total Dirichlet energy in heterophilous graphs compared with homophilous cases, which indicates that high-pass information needs more attention in heterophilous graphs. A deeper network is essential for such graphs, because it creates a larger receptive field during feature propagation and can thus accept more information from the nodes with the same label. From Figure~\ref{fig:DE_Layers}(f), we can observe that in a deeper GNN model, EE-UFG maintains a consistently higher Dirichlet energy with different $p/q$ ratios than GCN. 

\subsection{Node Classification Performance}
\paragraph{Real-world Datasets}We evaluate our proposed models for node classification tasks on nine real-world datasets: Texas, Wisconsin, Cornell introduced by \cite{pei2020geom}, Squirrel, Chalmeleon introduced by \cite{rozemberczki2021multi} and Cora, PubMeb, CiteSeer introduced by \cite{yang2016revisiting} and ogb-arxiv introduced by \cite{wang2020microsoft}. The homophily level ranges from 0.11 to 0.81, which measures the probability of connectivity between nodes with the same label in the graph. We test our model's performance on the public split~\cite{yang2016revisiting} and calculate the average test accuracy and standard deviation. For each dataset, all models are fine-tuned and tested on the same train/validation/test split.

\paragraph{Baselines} We select classic GNNs and state-of-the-art methods for heterophilous graphs and over-smoothing issue as our baselines: (1) classic GNN models: GCN \cite{kipf2016semi}, GAT \cite{velivckovic2017graph}, GraphSAGE \cite{hamilton2017inductive}, UFG \cite{zheng2021framelets}; (2) GNNs that can circumvent over-smoothing: GRAND \cite{chamberlain2021grand}, PairNorm \cite{zhao2019pairnorm}, GCNII \cite{chen2020simple}, EGNN \cite{zhou2021dirichlet}; (3) models for heterophilous graphs: FAGCN \cite{bo2021beyond}, MixHop \cite{abu2019mixhop}. We
use the official codes provided by the authors for all baselines. The hyper-parameter search space for EE-UFG is given in Appendix~\ref{appendix:hyper}. $\epsilon$ is a hyper-parameter in our architecture and we search the parameter space to get the optimal value, which might be different for different tasks. It is also demonstrated in Figure~\ref{fig:DE_Layers}(a) that Dirichlet energy is not sensitive to $\epsilon$.

\begin{table}[htb]
    \centering
    \setlength{\tabcolsep}{0.8mm}
    \resizebox{0.95\linewidth}{!}{
    \begin{tabular}{lcccccccccc}
        \toprule
         & \textbf{Texas}&\textbf{Wisconsin}&\textbf{Squirrel}&\textbf{Chameleon}&\textbf{Cornell}&\textbf{Ogb-arxiv}&\textbf{CiteSeer}&\textbf{PubMed}&\textbf{Cora}&\textbf{Rank}\\ \midrule
         Homophily level& $0.11$& $0.21$ & $0.22$ &$0.23$ &$0.30$ &$0.63$&$0.74$ &$0.80$ &$0.81$ \\ 
         \#Nodes&$183$ &$251$ &$5201$ &$2277$ &$183$ &$169343$& $3327$ &$18717$ & $2708$  \\
         \#Edges& $295$ &$466$ & $198493$ & $31421$ & $280$ & $1166243$ & $4676$ & $44327$ & $5278$\\
         \#Classes& $5$ & $5$ & $5$ & $5$ & $5$ &$40$& $7$& $3$ & $6$\\
         \midrule
         \textbf{GCN} & $55.1${\scriptsize$\pm 5.2$}&$51.8${\scriptsize$\pm 3.1$}& $53.2${\scriptsize$\pm 2.1$}&$64.8${\scriptsize$\pm 2.4$} &$60.5${\scriptsize$\pm 5.3$}&$71.7${\scriptsize$\pm 0.3$}&$71.9${\scriptsize$\pm 1.8$}& $78.7${\scriptsize$\pm 2.9$}&$81.5${\scriptsize$\pm 1.3$}&$6.9$\\
         
         \textbf{GAT}&$52.2${\scriptsize$\pm 6.6$}&$49.4${\scriptsize$\pm 4.1$}&$40.7${\scriptsize$\pm 1.5$}&$60.3${\scriptsize$\pm 2.5$}& $61.9${\scriptsize$\pm 3.1$}&$72.3${\scriptsize$\pm 0.9$}&$71.4${\scriptsize$\pm 1.9$}& $78.7${\scriptsize$\pm 2.3$}& $81.8$ {\scriptsize$\pm 1.3$} &$7.8$\\
         
         \textbf{GraphSAGE}&$\bf{82.4}${\scriptsize$\pm 6.1$}&$81.2${\scriptsize$\pm 5.6$}&$41.6${\scriptsize$\pm 0.7$}&$58.7${\scriptsize$\pm 1.7$}& $76.0${\scriptsize$\pm 5.0$}&$71.5${\scriptsize$\pm 0.3$}&$71.6${\scriptsize$\pm 1.9$}& $77.4${\scriptsize$\pm 2.2$}& $79.2$ {\scriptsize$\pm 7.7$}&$6.7$\\
         
         \textbf{GRAND} &$ 75.7${\scriptsize$\pm3.3$}&$ 79.4${\scriptsize$\pm 3.6$}&$40.1 ${\scriptsize$\pm1.5$}&$54.7${\scriptsize$\pm2.5$}&$\underline{82.2}${\scriptsize$\pm7.1$}&$72.2${\scriptsize$\pm0.2$}&$\underline{74.1}${\scriptsize$\pm 1.7$}& $78.8${\scriptsize$\pm1.7 $}&$83.6${\scriptsize$\pm1.0$}&$5.7$\\
         
         \textbf{PairNorm} &$60.3${\scriptsize$\pm4.3$}&$48.4${\scriptsize$\pm6.1$}&$50.4${\scriptsize$\pm 2.0$}&$62.7${\scriptsize$\pm 2.8$}&$58.9${\scriptsize$\pm 3.2$}&$70.4${\scriptsize$\pm 1.3$}&$73.6${\scriptsize$\pm 1.5$}&$78.3${\scriptsize$\pm 0.4$}&$82.3${\scriptsize$\pm 1.0$}&$7.2$\\
         
         \textbf{GCNII} &$77.5${\scriptsize$\pm3.8$}&$80.4${\scriptsize$\pm3.4$}&$38.5${\scriptsize$\pm 1.6$}&$63.9${\scriptsize$\pm 3.0$}&$77.9${\scriptsize$\pm 3.8$}&$72.5${\scriptsize$\pm 0.3$}&$73.4${\scriptsize$\pm 0.6$}&$\underline{80.3}${\scriptsize$\pm 0.4$}& $\underline{85.5}${\scriptsize$\pm 0.5$}&$4.5$\\
         
         \textbf{EGNN} &$81.0${\scriptsize$\pm0.8$}&$\bf{88.6}${\scriptsize$\pm3.2$}&$48.3${\scriptsize$\pm2.3$}&$62.7${\scriptsize$\pm2.6$}&$\bf{83.8}${\scriptsize$\pm4.6$}&$\underline{72.7}${\scriptsize$\pm 1.2$}&$70.4${\scriptsize$\pm2.8$}& $80.1${\scriptsize$\pm 3.6$}&$\bf{85.7}${\scriptsize$\pm3.7$}&$\underline{3.3}$\\

         \textbf{FAGCN} &$\bf{ 82.4}${\scriptsize$\pm6.9$}&$82.9${\scriptsize$\pm7.9$}&$42.6${\scriptsize$\pm 0.8$}&$55.2${\scriptsize$\pm 3.2$}&$79.2${\scriptsize$\pm 3.2$}&$70.6${\scriptsize$\pm 0.8$}&$72.7${\scriptsize$\pm 0.8$}&$79.4${\scriptsize$\pm 0.3$}& $84.1${\scriptsize$\pm 0.5$}&$5.0$\\

         \textbf{MixHop} &$77.8${\scriptsize$\pm2.5$}&$75.4${\scriptsize$\pm4.9$}&$43.8${\scriptsize$\pm3.4$}&$60.5${\scriptsize$\pm3.5$}&$73.5${\scriptsize$\pm6.3$}&-&$ 71.4${\scriptsize$\pm0.6$}& $\bf{80.8}${\scriptsize$\pm 0.3$}&$81.9${\scriptsize$\pm1.2$}&$6.0$\\

         % \textbf{H2GCN} &$82.2${\scriptsize$\pm5.3$}&$ 85.8${\scriptsize$\pm4.2$}&$37.9${\scriptsize$\pm2.0$}&$59.4${\scriptsize$\pm2.0$}&$82.1${\scriptsize$\pm6.6$}&-&$\bf{76.9}${\scriptsize$\pm 1.7$}&$\bf{89.6}${\scriptsize$\pm0.3$}&$\bf{87.8}${\scriptsize$\pm1.3$} \\
         
         \textbf{UFG} &$79.3${\scriptsize$\pm2.8$}&$78.8${\scriptsize$\pm3.2$}&$\underline{53.3}${\scriptsize$\pm1.5$}&$\underline{66.9}${\scriptsize$\pm1.1$}&$75.3${\scriptsize$\pm1.1$}&$71.9${\scriptsize$\pm 0.1$}&$72.7${\scriptsize$\pm0.6$}&$79.7${\scriptsize$\pm 0.1$}&$83.6${\scriptsize$\pm 0.6$}&$4.4$\\\midrule
         
         \textbf{EE-UFG} (ours)&$\underline{82.3}${\scriptsize$\pm3.2$}&$\underline{85.3}${\scriptsize$\pm3.3 $}&$\bf{55.3}${\scriptsize$\pm1.3$}&$\bf{68.0}${\scriptsize$\pm0.9$}&$\underline{82.2}${\scriptsize$\pm2.8$}&$\bf{73.2}${\scriptsize$\pm 3.8$}&$\bf{74.2}${\scriptsize$\pm1.3$}&$79.4${\scriptsize$\pm0.9$}&$83.5${\scriptsize$\pm0.2$}&$\bf{2.2}$\\
         \bottomrule
    \end{tabular}}
    \caption{Node classification performance comparison. Best result in \textbf{bold} and second best \underline{underlined}. "-" denotes out of memory or inapplicable.}
    \label{tab:11}
    \vspace{-0.2cm}
\end{table}
\begin{table}[htb]
    \centering
    \setlength{\tabcolsep}{0.9mm}
    \resizebox{0.95\linewidth}{!}{
    \begin{tabular}{l|cccc|cccc|cccc|cccc}
        \toprule
         &\multicolumn{4}{c}{\textbf{Chameleon} (H=0.23)}&\multicolumn{4}{c}{\textbf{Cornell} (H=0.30)} &\multicolumn{4}{c}{\textbf{CiteSeer} (H=0.74)}&\multicolumn{4}{c}{\textbf{Cora} (H=0.81)}  \\ \midrule
         \#Layer& 2 & 8 & 16 & 32 & 2 & 8 & 16 & 32 & 2 & 8 & 16 & 32 & 2 & 8 & 16 & 32\\
         \midrule
          \textbf{GCN}&$\bf{63.2}$ & $58.9$ & $50.2$ & $32.4$& $\bf{60.5}$ &$56.4$ & $44.3$ &$28.9$ & $\bf{68.7}$ & $33.6$ & $28.7$& $23.1$& $\bf{81.5}$& $35.8$ & $28.5$ &$ 22.0$  \\
          \textbf{UFG} &$\bf{66.2}$ & $ 58.8$ & $53.4$ & $47.7$& $\bf{74.3}$ &$65.2$ & $58.4$ & $53.5$ & $\bf{71.3}$ &$51.2$ & $46.8$ & $40.4$ &  $75.1$ & $\bf{79.4}$ & $57.1$ &$39.1$ \\
           \textbf{PairNorm}& $\bf{62.4}$ & $54.1$ & $46.4$ & $33.7$&$50.3$ & $\bf{58.4}$ & $57.2$ & $57.9$ & $\bf{73.6}$ &$70.3$ & $58.4$ & $35.8$&  $74.5$ &$81.6$ &$\bf{82.3}$ & $60.3$  \\
           \textbf{GCNII} &$60.7$ & $\bf{62.5}$ & $58.7$ & $42.8$& $67.6$ & $63.2$ & $\bf{77.8}$ & $76.4$ & $68.2$ &$70.6$ & $72.9$ & $\bf{73.4}$&  $82.2$ & $84.2$ & $84.6$ & $\bf{85.4}$  \\ 
           \midrule
           \textbf{EE-UFG}& $66.2$ & $\bf{68.0}$ & $63.5$ & $63.5$&$75.0$ & $\bf{82.2}$ & $81.3$ & $79.2$ & $64.8$ & $73.6$ &$ \bf{73.8}$ &$72.4$&  $83.5$ & $82.4$ & $\bf{83.5}$ & $81.4$  \\
         \bottomrule
    \end{tabular}}
    \caption{Performance comparison for GCN, UFG and EE-UFG with fix number of layers on three citation network datasets. The best result of each model is highlighted in \textbf{Bold}.}
    \label{tab:22}
    \vspace{-0.6cm}
\end{table}

\paragraph{Results}
Table~\ref{tab:11} shows the performance comparison on nine node classification tasks. We can observe that for heterophilous tasks, EE-UFG obtains a great boost compared with baselines, by better extracting the high-frequency information of the node itself. Our model ranks top 2 over seven real-world datasets with $H<0.8$ that are moderately or highly heterophilous. Over-smoothing issue is especially detrimental in heterophilous graph tasks, where multi-hop and deeper GNNs are necessary. In heterophilous graphs, the aggregated information from adjacent nodes contains more high-frequency information, thus, the high-pass components of the node itself should be better focused. Besides, EE-UFG inherits multi-hop aggregation properties from the framelet transform, taking into account all hops in the multi-scale framelet representation, which is essential for heterophilous graphs. The experimental results emphasize our model’s advantage over heterophilous graphs. 

It is known that the performance of GNNs will rapidly decay as the layers are stacked too much. The GCN-row in Table~\ref{tab:22} verifies this phenomenon. We can observe from Table~\ref{tab:22} that the UFG suffers less over-smoothing than classic GCN, partly due to its adaptive filter learning in the frequency domain. Other baselines, such as PariNom, GCNII, alleviate the over-smoothing issue to some extent, which however sacrifices performance, especially for heterophilous graphs. Without Framelet Augmentation, the performance of GNNs may begin to drop before it reaches optimal performance. The effect of Framelet Augmentation here is to delay the performance decay so that it can achieve the best performance with an appropriate number of stacked layers. Our proposed EE-UFG can basically circumvent over-smoothing and achieve better and more stable performance as the number of layers increases. From Table~\ref{tab:22}, we can observe that our model with 32 layers can still perform better than the best performance of other baselines on heterophilous datasets (e.g., Chameleon and Cornell). Besides, we can see that the best performance of EE-UFG occurs at a deeper layer.

\section{Discussion and Extension}\label{discussion}
\paragraph{Framelet Systems on Manifold}
Framelet systems can be well applied to manifold signals, $f\in L^2(\mathcal{M})$. Akin to the graph Laplacian, for a given manifold $\mathcal{M}$, we consider its Laplace-Beltrami operator $\mathcal{L}_B$ which is defined as $\mathcal{L}_Bf = -\hbox{div}(\nabla f)$. $\mathcal{L}_B^L$ and $\mathcal{L}_B^H$ can then be defined similarly as Eq.~\ref{definition:laplacian} respectively. In general, our proposed framelet augmentation method can be naturally extended to any other (symmetric) Laplacian-based propagation rules, using the framelet theory on manifolds. More details about Framelet extension on manifolds are given in Appendix~\ref{appd:extension}.

\paragraph{Limitations} Framelet augmentation is based on a symmetric Laplacian and a symmetric adjacency matrix, which is the general case. However, for some specific cases in geometric deep learning, such as simplicial complexes, non-square boundary matrices are involved to relate the signals between simplices of different dimensions. In such cases, our framelet augmentation can not be implemented. We will consider framelet augmentation strategy for these cases in future work. Besides, the computational complexity of framelet transform is $\mathcal{O}(N^2(nJ+1))$ which is a bit high.

% \paragraph{Sheaf} A cellular sheaf introduced by \citep{JC13}. \citep{hansen2020sheaf} originally designed Sheaf Neural Networks with a fixed sheaf Laplacian. \citep{bodnar2022neural} developed the theory and provided rich experimenta results on a learnable sheaf Laplacian. A cellula sheaf over a graph is a mathematical object associating a space with each node and edge in the graph and a map between these spaces for each incident node-edge pair. Cellular sheaf theory gives the sheaf Laplacian \citep{HGh19} which show the underlying geometry of the graph. One can define accordingly sheaf convolutions and sheaf diffusion process. Using the same formula \eqref{eq: cheby} as the graph, but replacing the graph Laplacian by the sheaf Laplacian $\mathcal{L}_\mathcal{F}$ fwhere sheaf Laplacian is trainable with network parameters. We can define \emph{sheaf framelet convolution} as
% \begin{equation}
%     \begin{aligned}
%         \mathcal{W}_{0,J} &\approx \mathcal{T}_0(2^{-K+J-2}\mathcal{L}_\mathcal{F})\cdots\mathcal{T}_0(2^{-K}\mathcal{L}_\mathcal{F}),\\
%         \mathcal{W}_{r,1} &\approx \mathcal{T}_r(2^{-K}\mathcal{L}_\mathcal{F}),\\
%         \mathcal{W}_{r,j}&\approx \mathcal{T}_r(2^{-K+j-1}\mathcal{L})\mathcal{T}_0(2^{-K+j-2}\mathcal{L}_\mathcal{F})\cdots\mathcal{T}_0(2^{-K}\mathcal{L}_\mathcal{F}).
%     \end{aligned}
% \end{equation}

\section{Related Work}
\paragraph{Over-smoothing and Dirichlet Energy}
One of the widely known plights of GNNs is over-smoothing, which has been studied by \citep{li2019deepgcns,xu2018representation,klicpera2018predict,oono2019graph,nt2019revisiting}. The Dirichlet energy was commonly used in these studies. Explanation paying attention to the structure of Laplacian has been undergone by \citep{li2019deepgcns,luan2019break,oono2019graph,cai2020note}. A large part of the methods come from empirical techniques in graph convolutional layers, like reliving the adjacent matrix by sparsification \citep{rong2019dropedge}, scaling node representations to avoid features caught into the invariant regime~\cite{zhou2021dirichlet}, adding residual connections \citep{chen2020measuring,li2019deepgcns,xu2018representation}. Several other empirical methods have been studied recently, like weight normalization~\cite{zhao2019pairnorm}, edge dropout~\cite{rong2019dropedge}, etc. Many other attempts beyond the graph matrix analysis also emerged like GCON \citep{rusch2022graph} using ODE dynamics, GRAND~\cite{chamberlain2021grand} and PDE-GCN~\cite{eliasof2021pde} regarding GNNs as continuous diffusion processes. Besides, in the field of spectral analysis, GNNs' updating process can be viewed as tackling low-frequency information \citep{wu2019simplifying,bo2021beyond,bo2021beyond}. However, these empirical techniques lack a necessary theoretical guarantee. Another type of method is controlling Dirichlet energy to alleviate the over-smoothing issue, e.g., EGNN. However, Figure~\ref{fig:DE_Layers} (d) shows that EGNN suffers from Dirichlet energy instability in some heterophilous cases. In contrast, to our best knowledge, we are the first to theoretically prove the enhancement of Dirichlet energy, taking advantage of multi-scale graph representation.

\paragraph{Wavelet Analysis on Graphs}
\cite{crovella2003graph} firstly proposed a formal approach to spatial traffic analysis on the wavelet transform. Polynomials of a differential operator were used to build a multi-scale tight frame by \citep{maggioni2008diffusion}. \cite{wang2020tight} gave the tight framelets framework on manifolds, which was then extended to graphs by \citep{dong2017sparse,zheng2022decimated} with the fast decomposition and reconstruction algorithms on undecimated and decimated frames for graph signals. In the regime of signal processing, \cite{mallat1989theory} established a tree-based wavelet system with localization properties, which is a milestone in the multi-resolution analysis. \cite{gavish2010multiscale} apply harmonic analysis to semi-supervised learning and construct Haar-like bases for it. \cite{wang2020haar,li2020fast,zheng2020mathnet} used the Haar-like wavelets system \citep{chui2015representation} to cope with deep learning tasks.

\section{Conclusion}
In this work, we develop a framelet analysis on graphs and generalize the generic graph convolution to a framelet version. Due to the energy difference between the low-pass and high-passes, we originally propose framelet augmentation which is surprisingly discovered to increase the Dirichlet Energy associated with the graph and keep it at a high and steady value. In practice, we demonstrate the behavior of framelet features during the training and the effectiveness of framelet augmentation to relieve the over-smoothing problem. Experimental Results also show that the proposed EE-UFG achieve excellent performance on node classification tasks.

% For natbib users:
\bibliographystyle{unsrtnat}
\bibliography{reference}
% For bibLaTeX users:
% \printbibliography
\newpage 
\appendix

\section{Theoretical Support}
\subsection{Framelet Dirichlet Energy Conservation}\label{DE_proof}
Here, we give the proof of proposition and remark mentioned in Section~\ref{sec:3}.
\begin{prop}
The Dirichlet energy is conserved under framelet decomposition:
\begin{equation}
    E_{total}(X)=E(X).
\end{equation}
\label{DE_conservation_2}
\end{prop}
\textbf{Proof.}
Let $(r,j)\in\{(r,j)|r=1,\cdots,n;j=1,\cdots,J\}\cup\{(0,J)\}$,
\begin{equation*}
    \begin{aligned}
        E_{total}(X) &=
        \sum\limits_{r,j}X^{\top} \mathcal{W}_{r,j}^\top \Tilde{\Delta}\mathcal{W}_{r,j}X\\
        &=\sum\limits_{r,j}X^\top U^\top\Lambda_{r,j}UU^\top\Lambda UU^\top \Lambda_{r,j} UX\\
        &= \sum\limits_{r,j} X^\top U^\top\Lambda^2_{r,j}\Lambda UX\\
        &=\sum\limits_{r,j}\sum\limits_i (UX)_i^2 (\lambda_{r,j}^i)^2\lambda_i\\
        &=\sum\limits_i (UX)_i^2\lambda_i\\
        &= X^\top U^\top\Lambda UX\\
        &= E(X),
    \end{aligned}
\end{equation*}
where $(UX)_i$ is the $i$th component of $UX$, and $\lambda_i,\lambda_{r,j}^i$ are the eigen-values of $\Lambda$, $\Lambda_{r,j}$. The fifth equality is because of the relation $\sum_{r,j} (\lambda_{r,j}^i)^2=1$. $\Lambda_{r,j}$ as follows,
\begin{equation*}
    \begin{aligned}
        \Lambda_{0,J}&=\widehat{a}(2^{-K+J-1} \Lambda)\cdots\widehat{a}(2^{-K}\Lambda),\\
        \Lambda_{r,1}&=\widehat{b^{(r)}}(2^{-K+j-1}\Lambda),\\
        \Lambda_{r,j}&=\widehat{b^{(r)}}(2^{-K+j-1}\Lambda)\widehat{a}(2^{-K+j-2}\Lambda)\cdots\widehat{a}(2^{-K}\Lambda).
    \end{aligned}
\end{equation*}
Therefore, the Dirichlet energy is preserved after the framelet decomposition.
\bs

\begin{remark}
The dirichlet energy components $E_{r,j}(X):= X^\top \mathcal{W}_{r,j}^\top\Tilde{\Delta} \mathcal{W}_{r,j}X$ are controlled by $\Lambda_{r,j}^2$, the diagonal matrix given in Eq.~\ref{operators}, where $(r,j)\in\{(r,j)|r=1,\cdots,n;j=1,\cdots,J\}\cup\{(0,J)\}$.
\label{remark_2}
\end{remark}
\textbf{Proof.}
Using $E_{r,j}(X):= X^\top \mathcal{W}_{r,j}^\top\Tilde{\Delta} \mathcal{W}_{r,j}X$, we obtain
\begin{equation*}
   \min_{i} \{(\lambda_{r,j}^i)^2\}E(X) \le E_{r,j}(X) \le \max_{i} \{(\lambda_{r,j}^i)^2\}E(X)\le 4E,
\end{equation*}
where we use that the eigenvalues of the normalized Laplacian are in the range of $[0,2]$, $\lambda_{r,j}^i$ is the $i$th eigenvalue of $\Lambda_{r,j}$.
\bs

\subsection{Framelet Energy Gap}\label{energy_gap}
In this part, we prove the energy gap between low-pass and high-pass coefficients, with the specific Haar-type filters. We consider $L_2$ norm of feature $X$ as its energy.
\begin{prop}
In the framelet system of 2 scales ($j=1,,2$) and 1 high-pass ($r=1$) with Haar-type filters, the energy of the low-pass is larger than the sum of the energy of the high-passes, i.e. $\|W_1^1\|^2+\|W_2^1\|^2\leq \|V_0\|^2$.
\label{prop:energy_gap}
\end{prop}
\textbf{Proof.}
With the relations that
$\left\{\begin{matrix}\widehat{\alpha}(2\xi)=\widehat{a}(\xi)\widehat{\alpha}(\xi)\\ 
\widehat{\beta}(2\xi) = \widehat{b}(\xi)\widehat{\alpha}(\xi)
\end{matrix}\right.$ and $\left\{\begin{matrix}\widehat{a}(\xi)=\cos(\xi/2)\\ 
\widehat{b}(\xi)=\sin(\xi/2)
\end{matrix}\right.$, we obtain
\begin{equation*}
    \frac{\widehat{\beta}(2\xi)}{\widehat{\alpha}(2\xi)}=\frac{\widehat{b}(\xi)}{\widehat{a}(\xi)}=\tan\bigl(\frac{\xi}{2}\bigr).
\end{equation*} 
As the framelets constitute a tight frame, we have the Parseval identity $\|\widehat{W}_1^1\|^2 + \|\widehat{V}_0\|^2 = \|\widehat{X}\|^2$. Thus, we can obtain the explicit expression of $\widehat{\alpha}$ and $\widehat{\beta}$ as follows,
\begin{equation*}
        \widehat{\alpha}(\Lambda/2)=\cos(\Lambda/{8})\cos(\Lambda/{16}),\quad
        \widehat{\beta}({\Lambda}/{2})=\sin({\Lambda}/{8})\cos({\Lambda}/{16}),\quad
        \widehat{\beta}({\Lambda}/{4})=\sin({\Lambda}/{16}).
\end{equation*}
This implies the energy difference between low-pass and high-passes reads
\begin{equation}\label{eq:energy difference}
    \|\widehat{V}_0\|^2 - \|\widehat{W}_1^1\|^2 - \|\widehat{W}_2^1\|^2 = \|\widehat{X}\|^2\Bigl(\bigl\|\widehat{\alpha}({\Lambda}/{2})\bigr\|^2-\bigl\|\widehat{\beta}({\Lambda}/{2})\bigr\|^2-\bigl\|\widehat{\beta}({\Lambda}/{4})\bigr\|^2\Bigr)
\end{equation}
% $\|\widehat{V}_0\|^2 - \|\widehat{W}_1^1\|^2 - \|\widehat{W}_2^1\|^2 = \|\widehat{X}\|^2\|\widehat{\alpha}(\frac{\Lambda}{2})-\widehat{\beta}(\frac{\Lambda}{2})-\widehat{\beta}(\frac{\Lambda}{4})\|^2$.
The RHS of \eqref{eq:energy difference} equals to
$\hbox{cos}^2(\frac{\Lambda}{8})\hbox{cos}^2(\frac{\Lambda}{16})-\hbox{sin}^2(\frac{\Lambda}{8})\hbox{cos}^2(\frac{\Lambda}{16})-\hbox{sin}^2(\frac{\Lambda}{16})$. Since the eigenvalues of the normalized Laplacian are in the range of $[0,2]$, it can be easily verified that the above trigonometric function is always larger than zero.
This then gives $\|W_1^1\|^2+\|W_2^1\|^2\leq \|V_0\|^2$.
\bs

Figure~\ref{fig:L2_norm} shows the $L_2$ norms of low-pass, high-passes, and the sum of high-passes of datasets with different homophily levels. It empirically verified that there exists an energy imbalance between low and high-passes, which inspires our energy enhancement strategy.

\subsection{Dirichlet Energy Enhancement}\label{append:enhancement_proof}
Next, we show how the Dirichlet energy is enhanced with framelet augmentation.
\begin{prop}
For $\epsilon>0$, the total framelet Dirichlet energy is increased with low-pass adjacency matrix $\widehat{A}^L$ and high-pass adjacency matrix $\widehat{A}^H$, i.e., $E_{total}^{\epsilon}(X)>E_{total}(X) = E(X)$. 
\end{prop}
\textbf{Proof.}
\begin{equation*}
    \begin{aligned}
        E_{total}^{\epsilon}(X) &=
        \sum\limits_{r,j}E_{r,j}^{\epsilon}(X) + E_{0,J}^{\epsilon}(X)\\
        &=\sum\limits_{r,j}(\mathcal{W}_{r,j}X)^{\top} (\Tilde{\Delta}-\epsilon \widehat{D}^{-1})(\mathcal{W}_{r,j}X) + (\mathcal{W}_{0,J}X)^{\top} (\Tilde{\Delta}+\epsilon \widehat{D}^{-1})(\mathcal{W}_{0,J}X)\\
        &=\sum\limits_{r,j}X^{\top}U^\top\Lambda_{r,j}UU^\top(\Lambda-\epsilon \widehat{D}^{-1})UU^\top\Lambda_{r,j}UX+X^{\top}U^\top\Lambda_{0,J}UU^\top(\Lambda+\epsilon \widehat{D}^{-1})UU^\top\Lambda_{0,J}UX\\
        &=\left(\epsilon X^\top U^\top \widehat{D}^{-1}\Lambda_{0,J}^2UX - \sum\limits_{r,j}\epsilon X^\top U^\top\widehat{D}^{-1}\Lambda_{r,j}^2UX\right)\\
        &+ \left(X^\top U^\top \Lambda\Lambda_{0,J}^2UX + \sum\limits_{r,j} X^\top U^\top\Lambda\Lambda_{r,j}^2UX\right)\\
        &=\epsilon X^\top U^\top\widehat{D}^{-1}\Bigl(\Lambda_{0,J}^2-\sum\limits_{r,j}\Lambda_{r,j}^2\Bigr)UX+E(X).
    \end{aligned}
    % \label{rebalance_energy}
\end{equation*}
By Proposition~\ref{prop:energy_gap} and its specific framelet system, $\Lambda_{0,J}^2-\sum\limits_{r,j}\Lambda_{r,j}^2\geq 0$, thus, $E^{\epsilon}_{total}(X)\geq E(X)$.
\bs

\subsection{Asymptotic Behavior of EE-UFG}\label{append:lambda}
\begin{prop}
    Let $A$ be an $n\times n$ augmented adjacency matrix, which is (symmetric) positive definite. Let $\lambda_k(A)$ be the $k$-th largest eigenvalue of $A$ ($k=1,2,\cdots,n$), and $A(\epsilon)$ denote $A + \epsilon D$, where $D$ is a positive diagonal matrix. Then, $\lambda_k(A(\epsilon))$ increases monotonically with $\epsilon$ and the following relation holds:
    \begin{equation*}
        \lambda_k(A^L)\leq \lambda_k(A)\leq \lambda_k(A^H)\quad \hbox{for~}\epsilon\geq 0,
    \end{equation*}
    where $A^L$ and $A^H$ are low-pass and high-passes adjacency matrices as defined in Eq.~\ref{definition:adjacency}.
\end{prop}
\textbf{Proof.}
By Eq.~\ref{definition:adjacency}, we know that $A^L = \widehat{A}-\epsilon \widehat{D}^{-1}$ and $A^H = \widehat{A}+\epsilon \widehat{D}^{-1}$, where $\widehat{D}^{-1}$ is a positive diagonal matrix. For symmetric matrices, we have the Courant-Fischer min-max theorem:
$$\lambda_k(A) = \hbox{min}\{\hbox{max}\{R_A(x)|x\in U \hbox{ and }x\neq 0\}
| \hbox{ dim}(U)=k\}$$
with $R_A(x) = \frac{\left \langle Ax, x \right \rangle}{\left \langle x,x \right \rangle}$.
We have $R_{A+B}(x) =\frac{\left \langle (A+B)x, x \right \rangle}{\left \langle x,x \right \rangle} = \frac{\left \langle Ax, x \right \rangle}{\left \langle x,x \right \rangle}+\frac{\left \langle Bx, x \right \rangle}{\left \langle x,x \right \rangle}>\hbox{max}\{R_A(x)\}$, if $B$ is positively definite. Thus, we have $\lambda_k(A+\epsilon \widehat{D}^{-1})\geq \lambda_k(A)$. Similarly, $\lambda_k(A-\epsilon \widehat{D}^{-1})\leq \lambda_k(A)$. Therefore, $\lambda_k(A^L)\leq \lambda_k(A)\leq \lambda_k(A^H)$ holds when $\epsilon\geq 0$.
\bs

\begin{figure}
    \centering
    \includegraphics[width=0.75\textwidth]{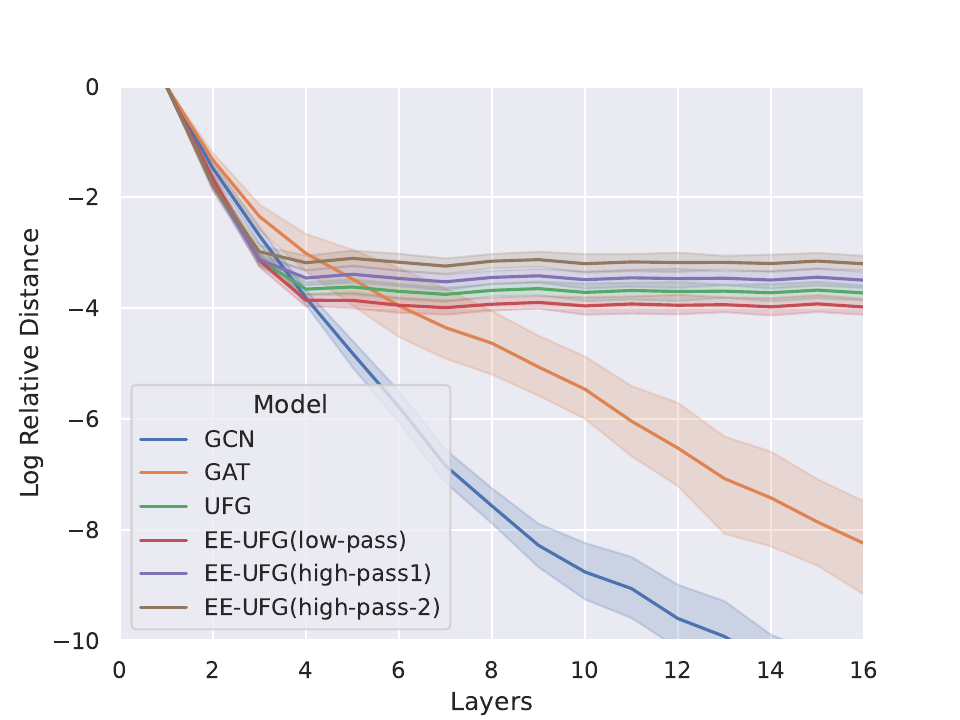}
    \caption{Layer-wise distances from the feature to the subspace $\mathcal{M}$. The result is an average of 100 runs. The Y-axis is the log relative distance, defined by $y^{(l)} = \textup{log}(d_{\mathcal{M}}(X^{(l)})/d_{\mathcal{M}}(X^{(0)}))$. $X^{(0)}$ is the initial feature representation, and $X^{(l)}$ is the output of the $l$-th layer.}
    \label{fig:Distance}
\end{figure}

Figure~\ref{fig:Distance} plots the logarithm of the relative distance from $l$-th layer's output to the subspace $\mathcal{M}$, i.e., $y^{(l)} = \textup{log}(d_{\mathcal{M}}(X^{(l)})/d_{\mathcal{M}}(X^{(0)}))$. The subspace $\mathcal{M}$ is defined by $\mathcal{M}=U\otimes\mathbb{R}^C
=\bigl\{\sum_{m=1}^M e_m\otimes w_m | w_m\in \mathbb{R}^C, e_m\in U\bigr\}\subseteq\mathbb{R}^{N\times C}$, where $U$ is the eigenspace associated with the
smallest eigenvalue (that is, zero) of a (normalized) graph Laplacian $\Delta$. We can observe that the low-pass of EE-UFG converges to a relatively closer distance to the subspace than the high-passes. It can also be predicted by Proposition~\ref{prop:lambda}. The layer-wise outputs of GCN and GAT exponentially approach the subspace $\mathcal{M}$. This subspace is invariant under any polynomial of the Laplacian Matrix, i.e., $\forall x\in \mathcal{M},\; g(\Delta)x\in \mathcal{M}$. It corresponds to the low-frequency part of graph spectra and only carries the information of the connected components of the graph \cite{oono2019graph}.

\subsection{Proof of Equivariance}\label{appendix:equivariance}
\begin{prop}
    An EEConv layer is permutation-equivariant.
    \label{prop:equivariance}
\end{prop}
\begin{proof}
Let $X\in \mathbb{R}^{n\times d}$ be the node features, the decomposition $\mathcal{W}=[\mathcal{W}_{0,2};\mathcal{W}_{1,1};\mathcal{W}_{1,2}]$, the augmented adjacency matrix $\mathbf{A}=[A^L; A^H; A^H]$, $P$ is the permutation matrix that is applied to the node features, and $\mathcal{W}^P$ the framelet transform of the permuted graph. Then, the following holds for any permutation matrix $P\in \mathbb{R}^{n\times n}$.
\begin{equation*}
    \mathbf{PAP^\top}\mathcal{W}^P PX=\begin{bmatrix}PA^LP^\top\mathcal{W}_{0,2}^P PX\\ PA^HP^\top\mathcal{W}_{1,1}^P PX\\ PA^HP^\top\mathcal{W}_{1,2}^P PX\end{bmatrix}
    =\begin{bmatrix}PA^LP^\top PU^\top\Lambda_{0,2}UP^\top PX\\ PA^HP^\top PU^\top\Lambda_{1,1}UP^\top PX\\ PA^HP^\top PU^\top\Lambda_{1,2}UP^\top PX\end{bmatrix}
    % =\begin{bmatrix}PA^L \mathcal{W}_{0,2}X\\ PA^H \mathcal{W}_{1,1}X\\ PA^H \mathcal{W}_{1,2}X\end{bmatrix}
    =\mathbf{P} \mathbf{A} \mathcal{W}X,
\end{equation*}
where $U$ is the graph Fourier transform, $\Lambda_{r,j}$ is defined in Eq.~\ref{operators}. Similarly, framelet reconstruction is permutation equivariant. Let $f$ be the function that represents EEConv in Eq.~\ref{modifed_framelet_convolution}, thus, for any permutation matrix $P$, we have $f(PH^{(l)})=Pf(H^{(l)})$, i.e., $f$ is permutation equivariant.
\end{proof}

\section{An Example: Enhancing Energy for Sheaf Convolution}\label{appd:extension}
Framelet systems can be well applied to manifold signals, $f\in L^2(\mathcal{M})$. Akin to the graph Laplacian, for a given manifold $\mathcal{M}$, we consider its Laplace-Beltrami operator $\mathcal{L}_B$ which is defined as $\mathcal{L}_Bf = -\hbox{div}(\nabla f)$. $\mathcal{L}_B^L$ and $\mathcal{L}_B^H$ can then be defined similarly as Eq.~\ref{definition:laplacian} respectively. In general, our proposed framelet augmentation method can be naturally extended to any other (symmetric) Laplacian-based propagation rules, using the framelet theory on manifolds.

The gradient operator $\nabla$ maps $x\in L^2(\mathcal{M})$ to its associated tangent plane $T_x(X)\in L^2(T\mathcal{M})$.

To obtain the discrete version, we can sample the manifold $\mathcal{M}$ at $N$ points, using polynomial-exact quadrature rules, like Gauss–Legendre quadrature sampling method. With these points, we can construct a set of triangular meshes ($V,E,F$). The edge connection between $i$ and $j$ indicates that $(i,j)\in E$ is shared by two triangular meshes, as originally proposed by \cite{bronstein2017geometric}. 
Using the same formula for the graph, but replacing the graph Laplacian by the Laplace-Beltrami operator $\mathcal{L}_B$, we can define \emph{Manifold framelet transforms} as
\begin{equation}
    \begin{aligned}
        \mathcal{W}_{0,J} &\approx \mathcal{T}_0(2^{-K+J-2}\mathcal{L}_B)\cdots\mathcal{T}_0(2^{-K}\mathcal{L}_B),\\
        \mathcal{W}_{r,1} &\approx \mathcal{T}_r(2^{-K}\mathcal{L}_B),\text{ and  }
        \mathcal{W}_{r,j}\approx \mathcal{T}_r(2^{-K+j-1}\mathcal{L})\mathcal{T}_0(2^{-K+j-2}\mathcal{L}_B)\cdots\mathcal{T}_0(2^{-K}\mathcal{L}_B).
    \end{aligned}
\end{equation}
\subsection{Diffusion Problems on Manifold}
The Laplace-Beltrami operator is closely related to the diffusion process over the manifold, which is governed by the PDE
$$\dot{f}(x,t) =-\mathcal{L}_Bf(x,t),\qquad f(x,0)=f_0(x), $$
where $f(x,t)$ is the signal at point $x$ at time $t$, $f_0(x)$ is the initial condition at point $x$. In the discrete setting, let $X^{(t)}\in \mathbb{R}^{N\times d}$ denote the feature representation at time point $t$. The following propagation rule can be used to approximate the continuous diffusion process on the manifold:
$$X^{(t+1)} =X^{(t)}-\mathcal{L}_BX^{(t)} = (I-\mathcal{L}_B)X^{(t)}.$$
Similar to \eqref{modifed_framelet_convolution}, a framelet manifold convolution for the diffusion process can be derived as follows.
For $r=1,\cdots,n$ and $j=1,\cdots,J$,
\begin{equation}
    \begin{aligned}
        H^{(l+1)}_{0,J}&=\sigma((I-\mathcal{L}_B^L)\mathcal{W}_{0,J}H^{(l)}W^{(l)}_{0,J}),\qquad
        H^{(l+1)}_{r,j}=\sigma((I-\mathcal{L}_B^H)\mathcal{W}_{r,j}H^{(l)}W^{(l)}_{r,j});\\
        H^{(l+1)} &= \mathcal{V}(H^{(l+1)}_{0,J};H^{(l+1)}_{1,1},\cdots,H^{(l+1)}_{n,J}),
    \end{aligned}
    \label{modifed_manifold_convolution}
\end{equation}
\subsection{Sheaf Laplacian}
 Here we discuss Sheaf Lapacians as an example of the generalization of our method to general manifolds. The definition of the framelet system on the sheaf, which we name as \textbf{\emph{sheaflets}}, is very similar to that on the graph but with a sheaf Laplacian which contains tunable parameters. Sheaflets can then be used to define \emph{sheaflet convolution} like \eqref{framelet_convolution} and then the enhanced sheaflet convolution as \eqref{modifed_framelet_convolution}. The latter can be proved to follow the same energy enhancement as the graph framelet convolution.

A cellular sheaf defined over a graph assigns each node and each edge a vector space and introduces a linear map between the associated spaces of each node-edge pair. In the mathematical language, a cellular sheaf  $\mathcal{F}$ on an undirected graph $\mathcal{G}$ is given by 
\begin{enumerate}
    \item a vector space $\mathcal{F}(v)$ for each vertex $v$ of $\mathcal{G}$,
    \item a vector space $\mathcal{F}(e)$ for each edge $e$ of $\mathcal{G}$,
    \item a linear map $\mathcal{F}_{v\triangleleft e}$: $\mathcal{F}(v)\rightarrow\mathcal{F}(e)$ for each incident vertex-edge pair $v\triangleleft e$ of $\mathcal{G}$.
\end{enumerate}
Construct the \textit{Sheaf Laplacian} $L_{\mathcal{F}}: C^0(\mathcal{G},\mathcal{F})\rightarrow C^0(\mathcal{G},\mathcal{F})$, where the diagonal blocks are $L_{\mathcal{F}_{vv}} = \sum_{v\triangleleft e}\mathcal{F}_{v\triangleleft e}^{\top}\mathcal{F}_{v\triangleleft e}$ and the non-diagonal blocks $L_{\mathcal{F}_{vu}}=-\mathcal{F}_{v\triangleleft e}^{\top}\mathcal{F}_{u\triangleleft e}$. Compared with the graph Laplacian, sheaf Laplacian is a $Nd\times Nd$ matrix, consisting of a class of linear operators over the graph, thus allowing the more underlying geometric and algebraic structure of the graph. $N$ is the number of nodes of $\mathcal{G}$, $d$ is the dimension of the stalks that associated to each node.

\subsection{Sheaflets} 
Let $\{(u_l,\lambda_l)\}_{l=1}^{Nd}$ the eigen-pair for the sheaf Laplacian $L_{\mathcal{F}}$ on $l_2(\mathcal{G}).$ For $j\in \mathbb{Z}$ and $p\in V,$ the \textit{undecimated sheaflets} $\phi_{j,p}(v)$ and $\psi_{j,p}^r(v), v\in V$ at scale $j$ are \textit{filtered Bessel kernels}
\begin{equation}\label{kernel}
\begin{aligned}
    &\phi_{j,p}(v):= \sum_{l=1}^{Nd} \widehat{\alpha}\left(\frac{\lambda_l}{2^j}\right ) \overline{u_l(p)}u_l(v),\\
    &\psi_{j,p}^r(v):= \sum_{l=1}^{Nd} \widehat{\beta}\left(\frac{\lambda_l}{2^j}\right ) \overline{u_l(p)}u_l(v), \quad r = 1,\dots,n.
\end{aligned}
\end{equation}
Here, $j$ and $p$ in $\phi_{j,p} (v)$ and $\psi_{j,p}^r(v)$ indicate the “dilation” at scale $j$ and the “translation” at a vertex $p\in V.$ $\alpha(\cdot), \beta(\cdot)$ are the scaling functions as defined in Section~\ref{sec:framelets}. Let $J,J_1,J>J_1$ be two integers. An \textit{undecimated sheaflet system} ${\rm UFS}(\Psi,\eta;\mathcal{G})$ (starting from
a scale $J_1$) as a non-homogeneous, stationary affine system:
\begin{equation}\label{ufs}
    \begin{aligned}
        {\rm UFS}_{J_1}^J (\Psi,\eta)&= {\rm UFS}_{J_1}^J(\Psi,\eta;\mathcal{G})\\
        &:=\{\phi_{J_1,p}:p\in V \} \cup \{\psi_{j,p}:p\in V, j=J_1,\dots ,J\}_{r=1}^n.
    \end{aligned}
\end{equation}
The system ${\rm UFS}_{J_1}^J(\Psi,\eta)$ is then called an \textit{undecimated tight frame} for $l_2(\mathcal{G})$ and the elements in ${\rm UFS}_{J_1}^J(\Psi,\eta)$ are called \textit{undecimated tight sheaflets} on $\mathcal{G}$.

The \textit{sheaflet coefficients} $V_0, W_j^r \in \mathbb{R}^{Nd\times f}$ are defined as the inner-product of the sheaflet and the sheaf signal $X\in \mathbb{R}^{Nd\times f}$, where $f$ denotes the feature dimension. The size of $V_0$ and $W_j^r$ is the same as the sheaf signal $X$. Then,
\begin{equation}
    V_0 =\left \langle\phi_{0,\cdot}, X\right \rangle = U^\top \widehat{\alpha}\bigl(\frac{\Lambda}{2}\bigr)UX\hbox{~~and~~} W_j^r= \left \langle\psi_{j,\cdot}^r, X\right \rangle=U^\top\widehat{\beta^{(r)}}\bigl(\frac{\Lambda}{2^{j+1}}\bigr)UX,
    \label{coefficients}
\end{equation}
where the scaling functions on $\mathcal{G}$ are as follows,
\begin{equation*}
    \widehat{\alpha}\bigl(\frac{\Lambda}{2^{j+1}}\bigr)= \mathrm{diag}\left(\widehat{\alpha}\bigl(\frac{\lambda_1}{2^{j+1}}\bigr),\cdots,\widehat{\alpha}\bigl(\frac{\lambda_{Nd}}{2^{j+1}}\bigr)\right),\quad
    \widehat{\beta^{(r)}}\bigl(\frac{\Lambda}{2^{j+1}}\bigr)
    =\mathrm{diag}\left(\widehat{\beta^{(r)}}\bigl(\frac{\lambda_1}{2^{j+1}}\bigr),\cdots,\widehat{\beta^{(r)}}\bigl(\frac{\lambda_{Nd}}{2^{j+1}}\bigr)\right).
\end{equation*}
\subsection{Implementation Format}
To reduce the computational complexity caused by eigendecomposition for Sheaf Laplacians, we use Chebyshev polynomials to approximate. Consider Chebyshev polynomials $\mathcal{T}_0,\cdots, \mathcal{T}_n$ of fixed degree $t$, and filter $a\approx \mathcal{T}_0$ and $b^{(r)}\approx\mathcal{T}_r$, then the above \ref{operators} can be approximated
\begin{equation*}
    \begin{aligned}
        \mathcal{W}_{0,J} &\approx U^\top\mathcal{T}_0(2^{-K+J-1}\Lambda)\cdots\mathcal{T}_0(2^{-K}\Lambda)U=\mathcal{T}_0(2^{K+J-2}L_{\mathcal{F}})\cdots\mathcal{T}_0(2^{-K}L_{\mathcal{F}}),\\[1mm]
        \mathcal{W}_{r,1} &\approx U^\top\mathcal{T}_r(2^{-K}\Lambda)U=\mathcal{T}_r(2^{-K}L_{\mathcal{F}}),\\[1mm]
        \mathcal{W}_{r,j}&\approx U^\top\mathcal{T}_r(2^{-K+j-1}\Lambda)\mathcal{T}_0(2^{-K+j-2}\Lambda)\cdots\mathcal{T}_0(2^{-K}\Lambda)U\\
        &=\mathcal{T}_r(2^{K+j-1}L_{\mathcal{F}})\mathcal{T}_0(2^{K+j-2}L_{\mathcal{F}})\cdots\mathcal{T}_0(2^{-K}L_{\mathcal{F}}).
    \end{aligned}
\end{equation*}
$\mathcal{L}_F$ is the sheaf Laplacian.

\subsection{From Sheaf Convolution to Sheaflet Convolution}
Sheaf convolution \cite{bodnar2022neural, hansen2020sheaf} is defined as follows,
\begin{equation}\label{eq:sheaf conv}
    Y = \sigma((I_{Nd}-L_\mathcal{F})(I_N\otimes  W_1)XW_2)\in \mathbb{R}^{Nd\times f_2},
\end{equation}
where $(I_N\otimes W_1)XW_2 = \Tilde{X}\in \mathbb{R}^{Nd\times f_2}$. Inheriting the characteristics of sheaf convolution in Eq.~\eqref{eq:sheaf conv}, we define \emph{Sheaflet Convolution} based on the sheaf framelet system as
\begin{equation}
    \begin{aligned}
        Y_{0,J}&=\sigma((I_{Nd}-\mathcal{L}_F)(I_N\otimes  W_1)\mathcal{W}_{0,J}XW_{0,J}),\\
        Y_{r,j}&=\sigma((I_{Nd}-\mathcal{L}_F)(I_N\otimes  W_1)\mathcal{W}_{r,j}XW_{r,j}),\\
        Y &= \mathcal{V}(Y_{0,J};Y_{1,1},\cdots,Y_{n,J}).
    \end{aligned}
    \label{sheaflet_convolution}
\end{equation}

\subsection{Dirichlet Energy for Sheaflets}
Applying our Dirichlet energy enhancement strategy to Sheaflet Convolution (Eq.~\ref{sheaflet_convolution}), we obtain the following \emph{Energy Enhanced Sheaflet Convolution},
\begin{equation}
    \begin{aligned}
        Y_{0,J}&=\sigma((I_{Nd}-\mathcal{L}_F^L)(I_N\otimes  W_1)\mathcal{W}_{0,J}XW_{0,J}),\\
        Y_{r,j}&=\sigma((I_{Nd}-\mathcal{L}_F^H)(I_N\otimes  W_1)\mathcal{W}_{r,j}XW_{r,j}),\\
        Y &= \mathcal{V}(Y_{0,J};Y_{1,1},\cdots,Y_{n,J}),
    \end{aligned}
    \label{modifed_sheaflet_convolution}
\end{equation}
where $\mathcal{L}_F^L$ and $\mathcal{L}_F^H$ are defined similarly as Eq.~\ref{definition:laplacian}.

The sheaflet Dirichlet energy with modified sheaf Laplacian $(\epsilon>0)$ is defined as 
\begin{equation}\label{eq:ee sheaflet}
\begin{aligned}
    E_{0,J}^{\epsilon}(X)=((I_N\otimes  W_1)\mathcal{W}_{0,J}X)^{\top} (\mathcal{L}_F+\epsilon D^{-1})((I_N\otimes  W_1)\mathcal{W}_{0,J}X)\\
    E_{r,j}^{\epsilon}(X)=((I_N\otimes  W_1)\mathcal{W}_{r,j}X)^{\top} (\mathcal{L}_F-\epsilon D^{-1})((I_N\otimes  W_1)\mathcal{W}_{r,j}X),
\end{aligned}
\end{equation}
where $D$ is the degree matrix of the sheaf Laplacian. The total sheaflet Dirichlet energy $E^{\epsilon}_\mathcal{F}(X) = E_{0,J}^{\epsilon}(X) + \sum_{r,j}E_{r,j}^{\epsilon}(X)$. The original sheaf Dirichlet energy is defined as $$E_{\mathcal{F}}(X)=((I_N\otimes  W_1)X)^\top\mathcal{L}_F((I_N\otimes  W_1)X).$$
It can be similarly proved that $E^{\epsilon}_\mathcal{F}(X)>E_{\mathcal{F}}(X)$ when $\epsilon>0$.

\section{Experimental Details}
\subsection{Experimental Setting}
The implementation of our model and training is based on PyTorch \cite{paszke2019pytorch} on NVIDIA Tesla A100 GPU with 6,912 CUDA cores and 80GB HBM2 mounted on an HPC cluster. PyTorch Geometric Library \cite{fey2019fast} is employed for all the benchmark datasets and baseline models. For each model, we run 2000 epochs for ogb-arxiv and 300 epochs for other datasets and select the configuration with the highest validation accuracy. The results in Table~\ref{tab:11} and Table~\ref{tab:22} are the average performance of each model over 10 fixed public splits.

\subsection{Datasets}\label{append:datasets}
We conduct experiments over 8 node classification datasets in 3 types:
\begin{enumerate}
    \item \textbf{Citation Network:} The Cora, CiteSeer, PubMed are citation network datasets introduced by \cite{yang2016revisiting}, where nodes represent documents in the computer science fields and edges represent citation links. 
    \item \textbf{Webpage Network:} The Texas, Wisconsin, and Cornell are webpage network datasets introduced by \cite{pei2020geom}. Nodes are the web pages and edges are the hyperlinks between them. Node features are bag-of-words representations of web pages. Nodes are classified into one of five categories: Students, Projects, Courses, Faculty and Staff.
    \item \textbf{Wikipedia Network:} The Chameleon and Squirrel are Wikipedia network datasets, introduced by \cite{rozemberczki2021multi}. Nodes are the web pages and edges are the hyperlinks between them. Node features represent several informative nouns on Wikipedia pages.
    \item \textbf{Ogb-arxiv:} The ogb-arxiv dataset is a citation network of Computer Science arxiv papers introduced by~\cite{wang2020microsoft}. Each node represents a paper and each directed edge indicates that one paper cites another one. Each node has a 128-dimensional feature that is averaged from the embeddings of words in its title and abstract. The task is to predict the 40 categories of the arXiv CS papers, which used to be labeled manually. However, with the increasing volume of CS papers, it is necessary to develop an automatic classification model. 
\end{enumerate}
Citation Network, Webpage Network and Wikipedia Network datasets are available at \url{https://pytorch-geometric.readthedocs.io/en/latest/modules/datasets.html}. 
Ogb-arxiv dataset is available at \url{https://ogb.stanford.edu/docs/nodeprop/#ogbn-arxiv}.

\begin{figure}
    \centering
    \includegraphics[width=0.9\textwidth]{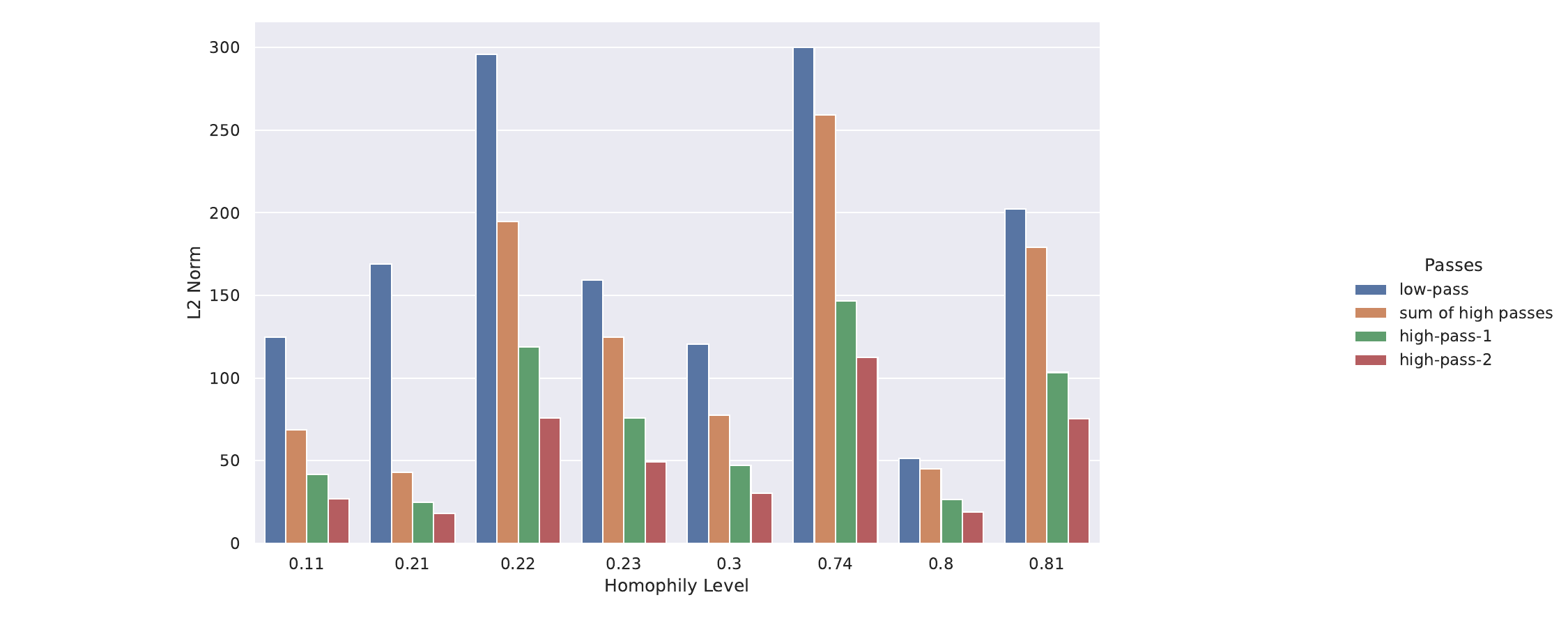}
    \vspace{-2mm}
    \caption{Energy ($L_2$ norm) of framelet coefficients for the 8 datasets.}
    \label{fig:L2_norm}
\end{figure}

\begin{figure}[h]
    \centering
    \includegraphics[width=0.7\textwidth]{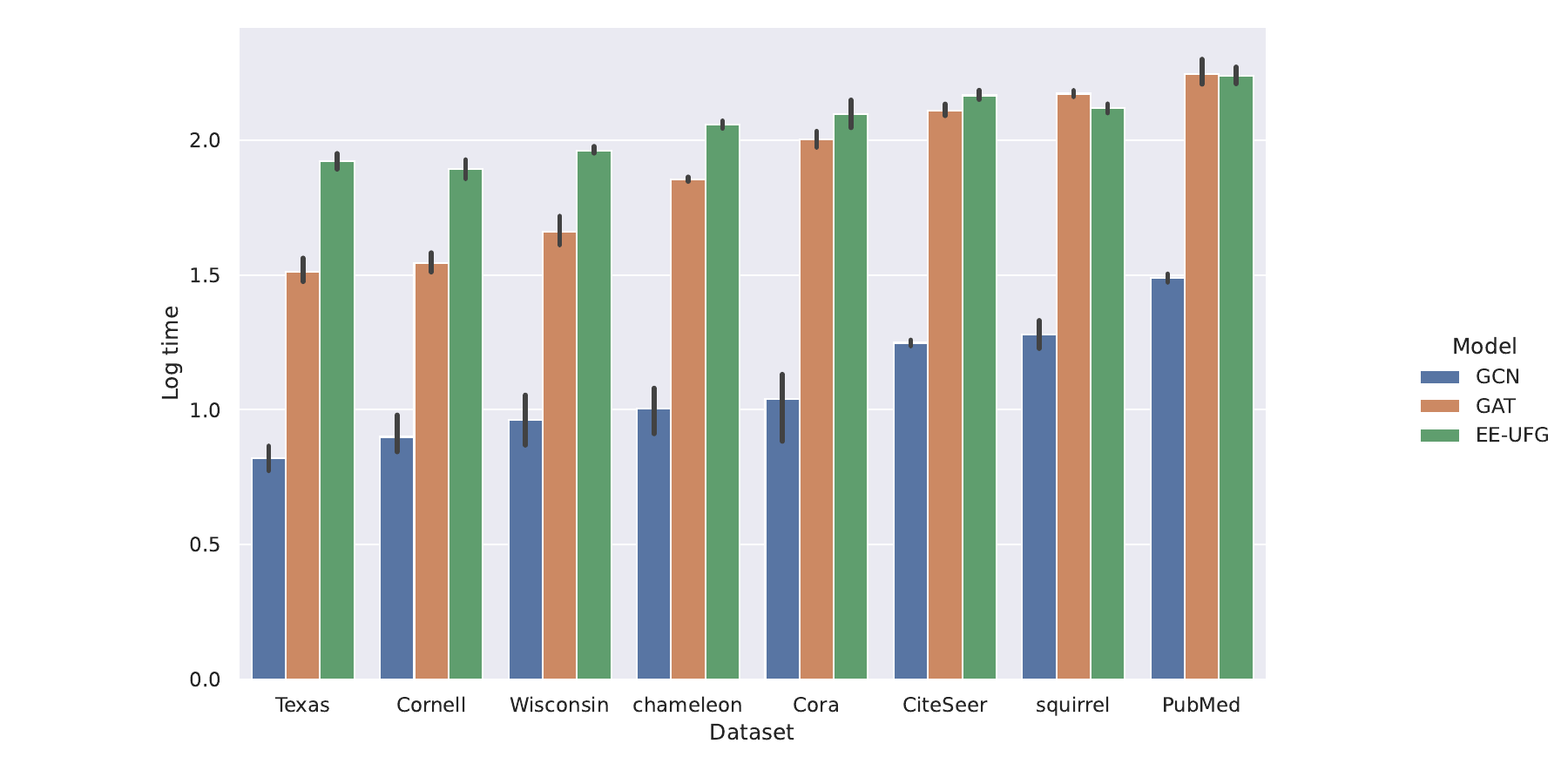}
    \vspace{-4mm}
    \caption{Running time comparison}
    \label{fig:complexity}
\end{figure}

\subsection{Baselines Selection} We select classic GNNs and state-of-the-art methods for heterophilous graphs and over-smoothing issues as our baselines: (1) classic GNN models: GCN \cite{kipf2016semi}, GAT \cite{velivckovic2017graph}, GraphSAGE \cite{hamilton2017inductive}, UFG \cite{zheng2021framelets}; (2) GNNs that can circumvent over-smoothing: GRAND \cite{chamberlain2021grand}, PairNorm \cite{zhao2019pairnorm}, GCNII \cite{chen2020simple}, EGNN \cite{zhou2021dirichlet}; (3) models for heterophilous graphs: FAGCN \cite{bo2021beyond}, MixHop \cite{abu2019mixhop}. For each model, we use the official codes provided by the authors. All models use the same train/validate/test split for a fair comparison. We notice that there were some impressive performance over Cora/CiteSeer/PubMed reported in the previous literature, like H2GCN~\cite{zhu2020beyond}, Geom-GCN~\cite{pei2020geom}, GGCN~\cite{yan2021two}. We do not adopt them as baseline models here because they randomly generate the train/validate/test split, which is different from our experimental setting.

\subsection{Computational Complexity}\label{append:complexity}
The time complexity of the algorithm is important for real-world deployment, especially for extremely large graph data. The framelet transform is equivalent to left-multiplying a specific transformation matrix. We stack the transformation matrices to obtain a tensor-based framelet transform with the computational complexity of $\mathcal{O}(N^2(nJ+1)d)$. $N$ is the number of nodes, $d$ is the feature dimension, $n$ is the number of high-pass filters and $J$ is the scale level of the low-pass. In our implementation, we fix $n=1$, $J=2$. Benefiting from an efficient message passing operator in PyTorch Geometry, we construct a large sparse adjacency matrix and stack all the passes, thus, the message passing in all passes can be executed in parallel.

Figure~\ref{fig:complexity} plots the running time on eight datasets we used. The number of nodes increases sequentially from left to right on the X-axis. The Y-axis is the logarithm of running time (in seconds). Each model has the same configuration, including hidden units, number of layers, etc., and run 300 epochs. We can observe from the figure summary that EE-UFG has a computational complexity close to GAT, especially when the number of nodes is large.

\begin{table}[th]
\centering
% \vspace{-0.2cm}
\setlength{\tabcolsep}{2mm}
% \resizebox{0.6\linewidth}{!}
{
\begin{tabular}{lr}
    \toprule
    Parameter & Search Space\\
    \midrule
    Learning rate &$[1\times 10^{-5}, 1\times 10^{-1}]$\\
    Hidden units &$\{16,32,64,128\}$\\
    Number of layers & $[1,10]$\\
    Epsilon &$[1\times 10^{-5}, 1\times 10^{-1}]$\\
    Dropout rate & $\{0.2,0.4,0.6,0.8\}$\\
    Weight decay &$[5\times 10^{-3},1\times 10^{-2}]$\\
    \bottomrule
  \end{tabular}}\vspace{0.5cm}
  \caption{Hyper-parameter Search Spaces of EE-UFG}\label{tab:33}
\end{table} 

\subsection{Hyper-parameter and Model Implementation}\label{appendix:hyper}
We employ Adam \cite{kingma2014adam} as our optimizer, and \texttt{ASHAScheduler} \cite{li2018massively} as our scheduler. Each model is fine-tuned with \texttt{Ray} \cite{liaw2018tune}. Table~\ref{tab:33} provides the hyper-parameter search space for reproduction. All baseline models are implemented with the official codes released by the authors.

\end{document}